\documentclass[journal]{IEEEtran}

\ifCLASSINFOpdf

\else

\fi

\usepackage{array}

\usepackage[T1]{fontenc}
\usepackage{lmodern}

\usepackage{amsmath,amsfonts,bm}

\def\eqref#1{equation~\ref{#1}}

\def\1{\bm{1}}

\DeclareMathAlphabet{\mathsfit}{\encodingdefault}{\sfdefault}{m}{sl}
\SetMathAlphabet{\mathsfit}{bold}{\encodingdefault}{\sfdefault}{bx}{n}

\newcommand{\bfA}{{\bf A}}

\newcommand{\bfD}{{\bf D}}

\newcommand{\bfI}{{\bf I}}
\newcommand{\bfJ}{{\bf J}}
\newcommand{\bfK}{{\bf K}}

\newcommand{\bfP}{{\bf P}}
\newcommand{\bfQ}{{\bf Q}}

\newcommand{\bfT}{{\bf T}}
\newcommand{\bfU}{{\bf U}}
\newcommand{\bfV}{{\bf V}}
\newcommand{\bfW}{{\bf W}}
\newcommand{\bfX}{{\bf X}}

\newcommand{\bfZ}{{\bf Z}}

\newcommand{\bfa}{{\bf a}}

\newcommand{\bfe}{{\bf e}}
\newcommand{\bfh}{{\bf h}}

\newcommand{\bfs}{{\bf s}}
\newcommand{\bfx}{{\bf x}}
\newcommand{\bfy}{ {\bf y}}

\newcommand{\bfp}{{\bf p}}

\newcommand{\bfz}{{\bf z}}

\newcommand{\bftheta}{{\boldsymbol \theta}}

\newcommand{\bfgamma}{{\boldsymbol \gamma}}
\newcommand{\bfGamma}{{\boldsymbol \Gamma}}
\newcommand{\bfbeta}{{\boldsymbol \beta}}

\newcommand{\bfxi}{{\boldsymbol \xi}}

\newcommand{\tr}[1]{\textcolor{red}{#1}}

\usepackage{hyperref}
\usepackage{booktabs}
\usepackage{url}
\usepackage{amsmath, amsthm, amssymb,graphicx}
\allowdisplaybreaks
\renewcommand{\eqref}[1]{(\ref{#1})}
\graphicspath{{Figures/}} 
\usepackage[dvipsnames]{xcolor}
\usepackage{cancel}
\usepackage{makecell}
\usepackage{multirow}
\usepackage{multicol}
\usepackage{cite}
\usepackage{cases}
\usepackage{enumitem}

\theoremstyle{plain}
\newtheorem{theorem}{Theorem}
\newtheorem{proposition}[theorem]{Proposition}
\newtheorem{lemma}[theorem]{Lemma}

\theoremstyle{definition}

\theoremstyle{remark}
\newtheorem{remark}[theorem]{Remark}

\usepackage{cleveref}
\crefname{proposition}{Proposition}{Propositions}
\Crefname{proposition}{Proposition}{Propositions}
\crefname{theorem}{Theorem}{Theorems}
\Crefname{theorem}{Theorem}{Theorems}
\crefname{lemma}{Lemma}{Lemmas}
\Crefname{lemma}{Lemma}{Lemmas}
\crefname{corollary}{Corollary}{Corollaries}
\Crefname{corollary}{Corollary}{Corollaries}
\crefname{definition}{Definition}{Definitions}
\Crefname{definition}{Definition}{Definitions}
\crefname{assumption}{Assumption}{Assumptions}
\Crefname{assumption}{Assumption}{Assumptions}
\crefname{remark}{Remark}{Remarks}
\Crefname{remark}{Remark}{Remarks}

\renewcommand{\footnoterule}{%
  \kern -3pt                         
  \hrule width 0.4\columnwidth       
  \kern 2.6pt                        
}

\newcounter{tmpTheoremCounter}

\newcommand{\settheoremnumber}[1]{%
  \setcounter{tmpTheoremCounter}{\value{theorem}}
  \setcounter{theorem}{\getrefnumber{#1}}%
  \addtocounter{theorem}{-1}%
}

\newcommand{\restoretheoremnumber}{%
  \setcounter{theorem}{\value{tmpTheoremCounter}}
}

\begin{document}
\title{Stability of Transformers under Layer Normalization}

\author{ Kelvin Kan, Xingjian Li, Benjamin J. Zhang, Tuhin Sahai, \\Stanley Osher, Krishna Kumar, 
Markos A. Katsoulakis
\thanks{K. Kan and SO are with UCLA, XL and K. Kumar are with UT Austin, BJZ is with UNC Chapel Hill, TS is with SRI International, MAK is with UMass Amherst.}}


\maketitle

\begin{abstract}
    Despite their widespread use, training deep Transformers can be unstable. 
    Layer normalization, a standard component, improves training stability, but its placement has often
    been \emph{ad hoc}.
    In this paper, we conduct a principled study on the forward (hidden states) and backward (gradient) stability of Transformers under different layer normalization placements. By utilizing discrete-time optimal control theory, our results provide key insights into the training dynamics: whether training drives Transformers toward regular solutions or pathological behaviors. 
    For forward stability, we derive explicit bounds on the growth of hidden states in trained Transformers. For backward stability, we analyze how layer normalization affects the backpropagation of gradients, thereby explaining the training dynamics of each layer normalization placement.
    Our analysis also guides the scaling of residual steps in Transformer blocks, where appropriate choices can further improve stability and performance. Our numerical results corroborate our theoretical findings. Beyond these results, our framework provides a principled way to sanity-check 
    the stability of Transformers under new architectural modifications, offering guidance for future designs.
\end{abstract}

\IEEEpeerreviewmaketitle

\section{Introduction}

Transformers have become the foundation of modern deep learning, driving state-of-the-art models across language, vision, and beyond. The training of Transformers, however, remains challenging due to instabilities in both their forward evaluation and gradient backpropagation, particularly as model depth grows. A key architectural component that improves the stability of Transformers is layer normalization (LN), which normalizes hidden states within each layer \cite{ba2016layernormalization,xiong2020layer}. While LN is widely applied in practice, its position within the Transformer block  is not rigorously justified, and has evolved largely through empirical testing and heuristics. 

Early Transformer designs used Post-LN, which applies LN \emph{after} adding the residual connection to the attention and feedforward modules. But this choice often exhibits suboptimal performance~\cite{xiong2020layer,kim2025peri} and requires delicate optimization scheduling~\cite{Popel2018TrainingTF,liu2019variance} to achieve stable training. Pre-LN, which places LN at the \emph{input} of the attention and feedforward modules, has since become the standard due to improved performance over Post-LN. Pre-LN, however, is known to produce \emph{excessively large} hidden states \cite{dettmers2022gptint,yu2024super,sun2024massive,fishman2025scaling,kim2025peri}, which can lead to numerical instability during training. Peri-LN, a newer alternative that places LN at \emph{both the input and output} of the attention and feedforward modules, has only recently been adopted in large-scale models due to its improved training stability and more regular hidden states compared to Pre-LN~\cite{kim2025peri}. However, its theoretical properties remain poorly understood, with recent work focusing on empirical studies. 

In this paper, we theoretically analyze how layer normalization placement affects Transformer stability during both evaluation and training. We examine forward stability --- the growth of hidden states of trained models --- and backward stability --- the regularity of gradients during backpropagation --- by casting the Transformer training formulation as a discrete-time optimal control problem. Our analysis explains the instability of Pre-LN and provides rigorous justification for the stability properties of Peri-LN observed in empirical studies. 

Using discrete-time optimal control theory, we show that the optimal solution for Pre-LN architectures grow unbounded in magnitude, while Peri-LN maintains controlled growth for entry- and data-wise moments. Specifically, we derive growth rates for hidden states that align with empirical observations for Peri-LN. Our backward stability analysis shows that for Pre-LN, gradients at individual layers grow proportionally with activations, which together with the unbounded hidden state growth results in training instability. In contrast, Peri-LN produces gradients that are invariant to the activation magnitude, implying stable gradients. These findings are then demonstrated through comprehensive numerical experiments in~\Cref{sec:experiments}, showing clear distinctions between different LN placements.

To summarize, our contributions are as follows:
\begin{itemize}
    \item We propose a novel theoretical framework grounded in discrete-time optimal control theory to study the stability of Transformers under different layer normalization placements. While prior work often focuses on models at initialization or relies on empirical evidence, our framework analyzes the trained models and provides a systematic assessment of whether the training drives Transformers toward regular solutions or pathologies. This assessment can also guide the design and evaluation of future Transformer architectures.
    \item We provide a novel derivation and rigorous analysis of this discrete-time optimal control framework. To the best of our knowledge, it is the first application of this theoretical approach to the analysis of Transformer training dynamics.
    \item We derive explicit bounds on Transformers' hidden state growth and analyze the training gradients under different layer normalization placements. Our theoretical results provide a principled explanation for empirical observations reported in the literature, which remain theoretically underexplored.
    \item Guided by our stability analysis, we introduce a residual step scaling and show theoretically that it improves both stability and performance in Peri-LN. We validate these improvements through experiments on language models from medium to large scales.
\end{itemize}

\section{Background and Setup}\label{sec:background}
In this section, we present the necessary background on Transformer architectures and their connection to discrete-time dynamics, which will underpin our analysis.
\paragraph{Notation} 
We use bold uppercase (e.g., $\bfX$) and lowercase letters (e.g., $\bfx$) to denote matrices and vectors, respectively.

\paragraph{Transformers} Let $\bfX_0 \in \mathbb{R}^{d \times n}$ be an (embedded and positionally encoded) input to a Transformer, where $d$ is the feature dimension, and $n$ is the number of tokens. The input is passed sequentially through a series of \emph{Transformer blocks}, such that the output of one block feeds into the next. Specifically, the $i$-th Transformer block reads
\begin{align}
\begin{split}
     &\bfX_{i+\frac{1}{2}} = \bfX_i + f_{\rm attn}(\bfX_i; \bftheta^{\rm attn}_i)  \\ & = \bfX_{i} + \sum_{h=1}^H  \bfW_i^h \bfV_i^h \bfX_i \, {\rm softmax} \left( \frac{(\bfK_i^h \bfX_i)^\top \bfQ_i^h \bfX_{i}}{\sqrt{k}} \right), \label{eq:self-attention}
     \end{split}
     \\
     &\bfX_{i+1} = \bfX_{i+\frac{1}{2}} + f_{\rm ffn}(\bfX_{i+\frac{1}{2}};\bftheta^{\rm ffn}_{i+\frac{1}{2}}), \label{eq:fully-connected}
\end{align}
for $i=0,1,...,D-1$. Here, $D$ is the total number of Transformer blocks. 
Each summand of the RHS of (\ref{eq:self-attention}) is called a \emph{self-attention head}, and the upper limit $H$ denotes the number of heads. The matrices $\bfQ_i^h,\bfK_i^h, \bfV_i^h \in \mathbb{R}^{k \times d}$ are commonly referred to as query, key, and value matrices, respectively, and $\bfW_i^h \in \mathbb{R}^{d \times k}$ is a weight matrix. All of these matrices are trainable and are collectively denoted as $\bftheta^{\rm attn}_i$. The module $f_{\rm ffn}$ is a feedforward network with parameters $\bftheta^{\rm ffn}_{i+\frac{1}{2}}$. It is applied separately to each token (i.e., each column of $\bfX_{i+\frac{1}{2}}$). Similarly, the softmax function is applied column-wise. It is noteworthy that both (\ref{eq:self-attention}) and (\ref{eq:fully-connected}) contain a \emph{skip connection}~\cite{he2016deep}.

The operation $f_{\rm attn}$ is known as a multi-head self-attention module, which is the key feature of Transformer architectures. This self-attention mechanism allows the model to dynamically focus on the most relevant parts of an input token sequence, enabling it to capture complex dependencies across both short- and long-range contexts. These capabilities make Transformers highly effective for tasks such as language modeling. Moreover, self-attention can be implemented efficiently: the underlying matrix operations can be parallelized, making Transformers particularly well-suited for long sequences (i.e., large $n$). 

After passing through the $D$ Transformer blocks, the Transformer's output $\tilde{\bfy}$ is computed as
\begin{equation}\label{eq:output_layer}
    \tilde{\bfy} = g(\bfX_{D}; \bfxi),
\end{equation}
where $g$ is either the composition of a decoder and a multilayer perceptron (MLP) or just an MLP, parametrized by $\bfxi$.

\paragraph{Discrete-time Dynamics.} \label{par:dynamics}
The skip connection structure in the Transformer blocks (\ref{eq:self-attention}-\ref{eq:fully-connected}) can be interpreted as a discrete-time dynamical system~\cite{haber2017stable,ruthotto2020deep,lu2020understanding}. The layer index $i\in [0,D]$ represents a discrete (artificial) time, and the Transformer hidden state $\bfX_i$ represents the state of the dynamics at time $i$. In particular, the initial condition is defined by the input $\bfX_0$ at $i=0$. The state then evolves according to the residual terms $f_{\rm attn}$ and $f_{\rm ffn}$ of (\ref{eq:self-attention}-\ref{eq:fully-connected}), which act as the velocity of the dynamics. Finally, the state at terminal time $i=D$ is the last Transformer hidden state $\bfX_D$. In this view, each Transformer block corresponds to evolving the hidden state $\bfX_i$ along the dynamics over two consecutive subintervals, one governed by the attention dynamics (\ref{eq:self-attention}) and the other by the feedforward dynamics (\ref{eq:fully-connected}).

\paragraph{Layer Normalization.} Layer normalization~\cite{ba2016layernormalization} is widely used in Transformer architectures. Given a hidden state $\bfX \in \mathbb{R}^{d \times n}$, the layer normalization operation\footnote{RMSNorm~\cite{zhang2019root} is sometimes used instead of LN. We discuss it in the Appendix, and our theory reamins the same under RMSNorm.} is applied to each of its tokens (columns) $\bfx \in \mathbb{R}^{d}$ as follows
\begin{equation}\label{eq:LN}
    {\rm LN}(\bfx; \bfgamma, \bfbeta) = \bfgamma \odot \hat{\bfx} + \bfbeta,
\end{equation}
where $\bfgamma, \bfbeta \in \mathbb{R}^d$ are trainable parameters, $\odot$ is the Hadamard element-wise product, and $\hat{\bfx}$ is given by\footnote{In practice, a small constant is added to the denominator of (\ref{eq:standardized}) to avoid division by zero; for simplicity, we omit this in our analysis.}
\begin{equation}\label{eq:standardized}
    \hat{\bfx} = \frac{\bfx - \mu}{\sigma}, \; \mu = \frac{1}{d} \sum_{l=1}^d x_l, \; \sigma = \sqrt{\frac{1}{d} \sum_{l=1}^d (x_l - \mu)^2}.
\end{equation}
Here, $x_l \in \mathbb{R}$ denotes the $l$-th entry of $\bfx$. Layer normalization is important because it regulates the magnitude of hidden states across layers. In particular, it is a known property that layer normalization projects its input onto an ellipsoid—a fact we state formally below in Lemma~\ref{lemma:ellipsoid}—providing a geometric explanation for how it prevents exploding activations in deep Transformers.

\begin{lemma}\label{lemma:ellipsoid}
    The layer normalization output $\bfz=\text{LN}(\bfx; \bfgamma, \bfbeta)$ lies on the ellipsoid 
    $$
    \mathcal{E} = \left\{ \mathbf{z} \in \mathbb{R}^d : (\mathbf{z} - \bfbeta)^\top \bfGamma^{-2} (\mathbf{z} - \bfbeta) = d \right\}, 
    $$
    where $\bfGamma = {\rm diag}(\bfgamma)\in \mathbb{R}^{d\times d}$.
\end{lemma}

\paragraph{Post-LN} An important design choice in Transformer architectures is the placement of layer normalization within each Transformer block. Early Transformer architectures use Post-LN, which applies normalization to the right-hand side of \eqref{eq:self-attention}-\eqref{eq:fully-connected} after the residual connection. Since Post-LN has been replaced by alternative layer normalization strategies in modern Transformers~\cite{takase-etal-2023-b2t,naver2025periln}, we focus on the alternatives in our analysis.

\paragraph{Pre-LN}
A prevalent choice is Pre-LN, which applies layer normalization to the inputs of modules. Specifically, given input $\bfX$, the output of the module under Pre-LN is given by\footnote{Here, LN applied to a matrix is understood as a column-wise operation.}
\begin{equation}\label{eq:Pre-LN}
    f^{\rm Pre}(\bfX) = f( {\rm LN}^{\rm in}(\bfX))
\end{equation}
where the module $f \in \{ f_{\rm attn}, f_{\rm ffn}\}$ is defined in (\ref{eq:self-attention})-(\ref{eq:fully-connected}). 
While Pre-LN can stabilize gradient during early training and reduce training time~\cite{xiong2020layer}, it is observed that the corresponding hidden states can grow exponentially over layers~\cite{sun2024massive,kim2025peri}. This can lead to exploding gradients and hence training instability, especially for deep models. While prior studies have been almost entirely empirical, our work provides a theoretical analysis of this phenomenon, explaining the underlying mechanisms and quantifying the effect.

\paragraph{Peri-LN}
Peri-LN is a recently adopted layer normalization placement. It applies layer normalization to both the inputs and outputs of the modules. In particular, the output of the module under Peri-LN is
\begin{equation}\label{eq:Peri-LN}
    f^{\rm Peri}(\bfX) = {\rm LN}^{\rm out} ( f( {\rm LN}^{\rm in}(\bfX)) ),
\end{equation}
where $f \in \{ f_{\rm attn}, f_{\rm ffn}\}$.
Peri-LN was deployed in major large-scale open-source models~\cite{kim2025peri}, including Olmo2~\cite{OLMo2}, Gemma2~\cite{Gemma2}, and Gemma3~\cite{team2025gemma}. Yet, their documentation provide little explanation of this choice. The study of~\cite{kim2025peri} reports that Peri-LN improves performance over Pre-LN. They provide an intuition that Peri-LN prevents gradient explosion and observe that Peri-LN yields hidden states with more regular magnitudes. However, their findings are mostly empirical, and their theoretical analysis is limited and does not explain this phenomenon. 

\section{Forward Stability of Transformers}\label{sec:forward}
In this section, we use analytical techniques from optimal control theory
 to analyze the forward stability of Transformer models under different layer normalization placements. These results corroborate empirical findings and provide a theoretical perspective on the effects of placement.

\subsection{Transformer Training Formulation and Mean-Field Control}\label{sec:OC_formulation}
We demonstrate that the Transformer training problem corresponds to a discrete-time mean-field control problem. Specifically, the Transformer training problem is given by
\begin{align}\label{eq:discrete_time_OC_obj}
    & \min_{\bftheta} \; \mathbb{E}_{(\bfX_0, \bfy)} \; G(\bfX_D, \bfy) \\ 
    \begin{split} &\text{s.t.} \; \bfX_{i+\frac{1}{2}} = 
    \begin{cases}\label{eq:discrete_time_OC_dynamics}
    \bfX_i + f^{\rm Pre}(\bfX_{i}; \bftheta_i), & \text{for Pre-LN},\\
    \bfX_i + f^{\rm Peri}(\bfX_i; \bftheta_{i}), & \text{for Peri-LN},
    \end{cases}
\end{split}
\end{align}
for $i=0,\frac{1}{2},1,...,D-\frac{1}{2}$.
Here, the expectation is taken over the input-output pairs $(\bfX_0, \bfy)$, and $f^{\rm Pre}$ and $f^{\rm Peri}$ are defined in \eqref{eq:Pre-LN} and \eqref{eq:Peri-LN}, respectively. The loss function $G$ measures the difference between the target output $\bfy$ and the model output $\tilde{\bfy}(\bfX_D)$ in (\ref{eq:output_layer}). For example, in classification and sequence generation tasks, the softmax loss is commonly used; in regression tasks, the mean squared error is used. 

The training formulation (\ref{eq:discrete_time_OC_obj}-\ref{eq:discrete_time_OC_dynamics}) can be interpreted as a discrete-time mean-field control problem~\cite{bensoussan2013mean}, where the loss function $G$ is the terminal cost, and the Transformer block (\ref{eq:discrete_time_OC_dynamics}) is the corresponding discrete-time dynamics; see~\Cref{sec:background} for more details.
Utilizing this connection, we apply optimal control theory to derive properties of the optimal solution (the trained Transformer). In particular, the mean-field control perspective allows us to study the well-definedness of the optimality conditions and characterize 
the properties of the trained Transformer. This provides a principled way to investigate how different design choices, such as layer normalization placements, influence the trained model’s behavior.

Our novel perspective differs from prior analyses, which often focus on models at initialization or rely on empirical observations; see~\Cref{sec:related_work} for related work. By \emph{studying the model at convergence}, we provide insights that directly correspond to the quality of learned representations. Importantly, this perspective allows us to assess \emph{whether training drives Transformers toward regular solutions or pathological behaviors}.

\subsection{Unbounded Growth of Pre-LN}\label{subsec:Pre_forward}
We demonstrate that under Pre-LN, the model hidden states are generally unbounded. The findings are summarized in the following theorem.

\begin{theorem}\label{thm:Pre-LN_illposed}
    The optimal solution $f^{\rm Pre}$ to the training problem (\ref{eq:discrete_time_OC_obj}-\ref{eq:discrete_time_OC_dynamics})  is unbounded in magnitude.
\end{theorem}

In essence, \Cref{thm:Pre-LN_illposed} states that the norm of the optimal solution $\|f^{\rm Pre}\|_F$ can be arbitrarily large.
Under these $f^{\rm Pre}$'s, the hidden state trajectories can be highly winding, or the hidden states can reach the target almost instantaneously and then remain steady. Moreover, the arbitrarily large $f^{\rm Pre}$ can lead to unbounded hidden states. These highly irregular hidden states can deteriorate the representations and thus the generalizability of the model \cite{zhang2023mean,kan2025optimal}. Critically, this also leads to numerical instability during training; see~\Cref{sec:backward}.

This phenomenon has been empirically observed in~\cite{sun2024massive,kim2025peri}, but limited theoretical explanations are provided. In contrast, our theory shows that even an optimal solution to the training problem can produce unbounded hidden states. This provides a theoretical basis for the common training instabilities reported in the literature.

A detailed proof of~\Cref{thm:Pre-LN_illposed} is given in the Appendix. The central argument is to examine the optimality conditions of the training problem, which are a discrete-time Hamilton-Jacobi-Bellman (HJB) partial differential equation (PDE) coupled with a mass transport equation that characterizes the evolution of the density of $\bfX_i$'s. 
Specifically, under Pre-LN and for each target output $\bfy$, the HJB PDE is given by
\begin{align}\label{eq:HJB_main}
\begin{split}
    &\Phi_{i, \bfy}(\bfX) -\Phi_{i+\frac{1}{2}, \bfy}(T_i(\bfX)) \\ &= - \left[ \langle \nabla \Phi_{i+\frac{1}{2}, \bfy}(T_i(\bfX)), f_i(\bfX) \rangle 
    + H(-\nabla \Phi_{i+\frac{1}{2}, \bfy}(T_i(\bfX))) \right], 
\end{split}
\end{align}
where $f_i(\bfX) = f^{\rm Pre}(\bfX; \theta_i)$, and $T_i(\bfX) = \bfX + f^{\rm Pre}(\bfX; \theta_i)$ denotes the operator for the $i$-layer defined in~\eqref{eq:discrete_time_OC_dynamics}, and $\Phi_{i, \bfy}$ represents the potential function, which characterizes the optimal solution to the training problem (\ref{eq:discrete_time_OC_obj}-\ref{eq:discrete_time_OC_dynamics}) at the $i$th layer and for the target output $\bfy$. Moreover the Hamiltonian is given by
\begin{equation}\label{eq:Pre_LN_Hamiltonian}
    H(-\nabla \Phi_{i+\frac{1}{2}, \bfy}(T_i(\bfX))) = \sup_{\bfV \in \mathbb{R}^{d \times n}} \langle -\nabla \Phi_{i+\frac{1}{2}, \bfy}(T_i(\bfX)), \bfV \rangle.
\end{equation}

The HJB PDE~\eqref{eq:HJB_main} is derived using dynamic programming principle. In the HJB PDE, the Hamiltonian \eqref{eq:Pre_LN_Hamiltonian} acts as a feedback control mechanism. By taking the supremum over $\bfV$, the Hamiltonian explicitly defines the optimal velocity field for the hidden states $\bfX_i$.

The Hamiltonian~\eqref{eq:Pre_LN_Hamiltonian} under Pre-LN is unbounded above and equals infinity. Subsequently, there is no well-defined HJB PDE and hence optimality conditions which characterize the solution. This implies that the training problem is ill-posed. In particular, it admits solutions with arbitrarily large velocity fields $f^{\rm Pre}$.

An intuitive remedy for the unboundedness in \Cref{thm:Pre-LN_illposed} is to apply weight decay during training, which is standard practice in training modern Transformers. Indeed, with weight decay, the Hamiltonian exists, and the HJB PDE and thus optimality conditions are well-defined. However, while weight decay gives rise to optimality conditions, it does not completely eliminate the issue: depending on the decay magnitude, 
the hidden states can still grow exponentially across layers, as formalized below. 

\begin{theorem}\label{thm:PreLN_expo_growth}
    For a Pre-LN Transformer trained with weight decay, given an input $\bfX_0$, the mean absolute value of the terminal hidden states ${\rm MA}(\bfX_D)$ satisfies
\begin{align}\label{eq:expo_growth_discrete}
        {\rm MA}(\bfX_D) \leq  \frac{\bigl(1 + C(\lambda)\bigr)^D}{\sqrt{nd}} \, \|\bfX_0\|_F = \mathcal{O}(e^D),
    \end{align}
    where $\| \cdot \|_F$ denotes the Frobenius norm, $C$ is a constant whose magnitude depends on the weight decay hyperparameter $\lambda$.
\end{theorem}
Although weight decay mitigates unboundedness, the exponential growth is still undesirable as it can lead to numerical instability. This is particularly problematic as model architectures become deeper, which is the current trend in industrial practice. Moreover, the strength of weight decay requires tuning to balance mitigation of growth and model performance.

In the next subsection, we establish that, in contrast, Peri-LN guarantees only linear growth of the hidden states and quadratic growth of their variance, offering more controlled dynamics.

\subsection{Controlled Growth of Peri-LN} 
The main difference between Pre-LN and Peri-LN is the placement of layer normalization on the module output. Intuitively, this normalization on the output prevents unbounded activations in~\Cref{thm:Pre-LN_illposed} from occurring. Indeed, output normalization restricts the velocity field for each token in the training problem~(\ref{eq:discrete_time_OC_obj}-\ref{eq:discrete_time_OC_dynamics}) to an ellipsoid (Lemma~\ref{lemma:ellipsoid}). Thus, the corresponding HJB PDE is given by \eqref{eq:HJB_main}, with $f_i(\bfX) = f^{\rm Peri}(\bfX; \theta_i)$, and $T_i(\bfX) = \bfX + f^{\rm Peri}(\bfX; \theta_i)$, and the Hamiltonian given by
\begin{equation}\label{eq:Peri_LN_Hamiltonian}
    H(-\nabla \Phi_{i+\frac{1}{2}, \bfy}(T_i(\bfX))) = \sup_{\substack{\bfV_{:, j} \in \mathcal{E}_i \\ j=1,...,n}} \langle -\nabla \Phi_{i+\frac{1}{2}, \bfy}(T_i(\bfX)), \bfV \rangle,
\end{equation}
where $\mathcal{E}_i  = \left\{ \mathbf{z} \in \mathbb{R}^d : (\mathbf{z} - \bfbeta_i^{\rm out})^\top \left(\bfGamma_i^{\rm out} \right)^{-2} (\mathbf{z} - \bfbeta_i^{\rm out}) = d \right\}$ is the ellipsoid defined by the Per-LN output layer normalization parameters at the $i$th layer. Note that the Hamiltonian~\eqref{eq:Peri_LN_Hamiltonian} is bounded and admits a bounded solution. This boundedness ensures that the HJB PDE and consequently the optimality conditions, are well-defined~\cite{zhang2023mean}. Subsequently, the enhanced well-posedness of the training problem leads to more robust and informative results for the behavior of the trained model.

We quantify the improved boundedness as follows.
\begin{theorem}[Controlled Growth of Entry-wise Moments]\label{thm:linear_growth_discrete}
    Given an input $\bfX_0$, the mean absolute value $\rm MA$ and variance $\rm Var$ of the terminal hidden states $\bfX_D$ of a Peri-LN Transformer satisfy, respectively,
    \begin{align}\label{eq:linear_growth_discrete}
        {\rm MA}(\bfX_D) &\leq \frac{1}{\sqrt{nd}} \| \bfX_0 \|_F + 2 D (\gamma_{\rm max} + \beta_{\rm max}) = \mathcal{O}(D),\\
        \begin{split}\label{eq:variance_entry_growth_discrete}
         {\rm Var}(\bfX_D) &\leq \frac{( \left\|  \bfX_{0} \right\|_F + 2D \sqrt{nd} (\gamma_{\rm max} + \beta_{\rm max}) )^2}{nd-1} \\ &= \mathcal{O}(D^2), 
    \end{split}
    \end{align}
where $\gamma_{\rm max} := \max\limits_{1 \leq i \leq D}  \{ \| \bfgamma^{\rm out}_{{\rm attn}, i} \|_\infty, \|\bfgamma^{\rm out}_{{\rm ffn}, {i+\frac{1}{2}}}\|_\infty \}$ takes the maximum over all layers and both attention and feedforward sublayers, with ``out'' denoting output layer normalization; similarly for $\beta_{\rm max}$.
\end{theorem}

The bounds establish that, for each input, the hidden state magnitude and variance grow at most linearly and quadratically. Moreover, we can also bound the variance of each entry across the dataset as follows.
\begin{theorem}[Quadratic Growth of Data-wise Variance]\label{thm:datawise_variance}
    Let $\bfX_0$ be an input and $\bfX_D$ its corresponding terminal hidden states. For each entry $x$ of $\bfX_D$, its variance across the data distribution satisfies
    \begin{equation}\label{eq:variance_across_data}
    \begin{split}
        {\rm Var}(x) &\leq \mathbb{E}_{\bfX_0} \left[ \left(  \|\bfX_0\|_F + 2D\sqrt{nd}(\gamma_{\rm max} + \beta_{\rm max}) \right)^2\right] \\
        &= \mathcal{O}(D^2),
    \end{split}
    \end{equation}
    where $ \gamma_{\rm max}$ and $\beta_{\rm max}$ are defined as in~\Cref{thm:linear_growth_discrete}, and the expectation is taken over $\bfX_0$ drawn from the data distribution.
\end{theorem}

The results \eqref{eq:linear_growth_discrete}-\eqref{eq:variance_across_data} confirm the empirical findings of~\cite{kim2025peri}, and provide the missing theoretical justification. We remark again that, in contrast, both the magnitude and variance of hidden states for Pre-LN Transformers are unbounded. 

Our bounds ensure that, under Peri-LN, the hidden states (i.e., the learned representations) remain well-conditioned and avoid exploding activations. This controlled growth is particularly important to preserve the quality of representation in training deep networks. It can prevent degradation or loss of information across many layers, an undesirable behavior arising from the unboundedness of Pre-LN.

\subsection{Uncertainty Quantification} 
We next analyze how Peri-LN Transformers propagate uncertainty in the input distribution to the distribution of terminal hidden states, enabling uncertainty quantification of the outputs.

\begin{theorem}\label{thm:uq}
    Let $\mu_0$ and $\nu_0$ be any two input distributions, $\mu_D$ and $\nu_D$ denote their pushforwards to the terminal hidden states under a Peri-LN Transformer, and $W^p_p(\mu,\nu) = \inf_{\gamma \in \Gamma(\mu,\nu)} \int_{\mathbb{R}^{dn} \times \mathbb{R}^{dn}} \|\bfX-\bfX'\|_p^p d\gamma $ to be the $p$-Wasserstein distance. There exists $\hat{C}(p)$ such that for any $p \geq 1$,
    \begin{equation}\label{eq:stable_forward_propagation_discrete}
        W_p(\mu_D,\nu_D) \leq 2^{\frac{p-1}{p}} \left( \hat{C}(p)W_p(\mu_0,\nu_0) +4 D \sqrt{nd}   \gamma_{\rm max}\right).
    \end{equation}
    Here, $ \gamma_{\rm max}$ is defined as in~\Cref{thm:linear_growth_discrete}. In contrast, for Pre-LN Transformers, the difference can be unbounded.
\end{theorem}

The bound~\eqref{eq:stable_forward_propagation_discrete} provides a quantitative bound: under Peri-LN, uncertainty in the input distribution, measured in Wasserstein distance, leads to a controlled level of uncertainty in the terminal distribution, up to an additive constant. The uncertainty can be due to new data, noise, or adversarial modifications. We remark that this bound represents a worst-case scenario; in practice, the difference between terminal distributions can be significantly smaller.
In contrast, Pre-LN may amplify amplify uncertainty in the input, leading to unbounded differences, as shown in~\Cref{thm:Pre-LN_illposed}. This contrast highlights the practical advantage of Peri-LN. 

Moreover, using techniques from distributionally robust optimization (DRO), we can show that~\eqref{eq:stable_forward_propagation_discrete} leads to non-asymptotic generalization bounds for in-distribution data. Using DRO, we can also derive bounds on expected test loss for distributions within a Wasserstein ball around the training distribution, providing theoretical insights into the performance on out-of-distribution data. Due to space constraints, we defer the discussion and derivations to the Appendix.

Notably, the literature on stable architectures achieves similar uncertainty quantification bounds through techniques other than layer normalization, such as explicit regularization~\cite{kan2025optimal} and constraints on model parameters~\cite{haber2017stable,ruthotto2020deep}. To the best of our knowledge, we are the first to analyze layer normalization from this perspective.

\section{Backward Stability of Transformers}\label{sec:backward}
While forward stability bounds hidden-state growth, backward stability concerns the propagation of gradients. In deep Transformers, the gradient magnitude depends on each block’s local sensitivity. We use this aspect to study \emph{the training dynamics} of Pre- and Peri-LN Transformers. We show that Peri-LN ensures stable backpropagation, whereas Pre-LN can produce unbounded gradients for large activations. This is crucial because once gradients explode, the training process becomes unstable and can effectively fail, wasting all computation up to that point.

The gradient of the loss function with respect to $\bftheta_i$, the weights of the $i$th block, is\footnote{We drop $G$'s dependency on $\bfy$ for notational simplicity.}
\begin{align}\label{eq:backprop_grad}
\nabla_{\bftheta_i} G(\bfX_D) 
= \nabla_{\bftheta_i} \bfX_{i+1} \cdot
  \bfJ_{i:D} \cdot
\nabla_{\bfX_D} G(\bfX_D),
\end{align}
where 
\begin{equation}\label{eq:gradient_product_theorem}
    \bfJ_{i:D} =  \prod_{j=i+1}^D \left( \bfI + \nabla_{\bfX_{j-1}} f(\bfX_{j-1}; \bftheta_{j-1})\right) ,
\end{equation}
$f \in \{ f^{\rm Pre}, f^{\rm Peri}\}$, and $\nabla_{\bfX_{j-1}} f(\bfX_{j-1}; \bftheta_{j-1})$ is the local sensitivity of the $j$-th block to its input.

The identity matrix in each factor of~\eqref{eq:gradient_product_theorem} arises from the skip connections. This ensures that, even with small local sensitivity, the gradient does not vanish~\cite{he2016deep}. Thus, vanishing gradients are alleviated, and the dominant concern is whether the sensitivity can be unbounded, leading to gradient explosion.

In the following, we analyze the behavior of the local sensitivity term and its potential to cause gradient explosion under different layer normalization schemes.

\paragraph{Gradient Explosion under Pre-LN} Under Pre-LN, the local sensitivity can become unbounded as the activations grow.

\begin{proposition}\label{thm:gradient_explosion}
Under Pre-LN, the sensitivity $\nabla_{\bfX_{j-1}} f^{\rm Pre}(\bfX_{j-1}; \bftheta_{j-1})$ grows proportionally with the activations ($f_{\rm ffn}$ and $f_{\rm attn}$ in (\ref{eq:self-attention})-(\ref{eq:fully-connected})). 
\end{proposition}
As a result, when large activations occur---which we show in~\Cref{subsec:Pre_forward} is possible---the product in \eqref{eq:gradient_product_theorem} can explode. This causes the gradient for preceding layers to be arbitrarily large and unstable during training.

\paragraph{Gradient Stability under Peri-LN.}
In contrast, we note that the local sensitivity is stable under Peri-LN even in the presence of large activations.

\begin{proposition}\label{thm:gradient_stability}
 Under Peri-LN, the sensitivity $\nabla_{\bfX_{j-1}} f^{\rm Peri}(\bfX_{j-1}; \bftheta_{j-1})$ is \textit{invariant} to the magnitude of the activation. 
\end{proposition}

Importantly, by Proposition~\ref{thm:gradient_stability}, even when a large activation occurs, the sensitivity remains at its nominal magnitude. This invariance is especially critical in deep networks, where $\bfJ_{i:D}$ involves a product of many terms: by having each term invariant, Peri-LN helps control the compounding effect that could otherwise lead to gradient explosion. This facilitates more stable backpropagation of training signals and improves the stability of training dynamics.

\section{Improved Stability via Scaled Residual Steps}\label{sec:delta_t}
Guided by our forward and backward stability analysis, we consider a generalized discrete-time dynamics with time increment $\Delta t$, which reads
\begin{align*}
     \bfX_{i+\frac{1}{2}} &= \bfX_i + \Delta t \cdot f_{\rm attn}(\bfX_i; \bftheta^{\rm attn}_i) 
     \\
     \bfX_{i+1} &= \bfX_{i+\frac{1}{2}} + \Delta t \cdot f_{\rm ffn}(\bfX_{i+\frac{1}{2}} ;\bftheta^{\rm ffn}_{i+\frac{1}{2}}). 
\end{align*}
Here, $f_{\rm attn}$ and $f_{\rm ffn}$ are defined in \eqref{eq:self-attention}-\eqref{eq:fully-connected}. 
Remark that $\Delta t=1$ recovers the standard Transformer blocks (\ref{eq:self-attention})-(\ref{eq:fully-connected}). We consider the case when the time increment is shorter $\Delta t<1$, which scales the magnitude of each residual update. 

\paragraph{Improved Forward Stability.} 
This modification improves the forward stability bounds for Peri-LN. For instance, it sharpens the bound on output uncertainty~\eqref{eq:stable_forward_propagation_discrete}: for two input distributions $\mu_0$ and $\nu_0$, the corresponding output distributions now satisfy
\begin{equation}\label{eq:forward_stable_improved}
        W_p(\mu_D,\nu_D) \leq 2^{\frac{p-1}{p}} \left( \hat{C}(p)W_p(\mu_0,\nu_0) +4 \Delta t D \sqrt{nd}   \gamma_{\rm max}\right).
    \end{equation}
Here, $\Delta t<1$ explicitly scales the constant term. 
The scaling reduces how each sub-layer amplifies differences, thereby limiting their growth.
This sharper bound leads to improved generalization bounds for both in-distribution and out-of-distribution; details are provided in the Appendix. 

In the same way, the modification also sharpens the bounds on hidden state growth~\eqref{eq:linear_growth_discrete}-\eqref{eq:variance_across_data} by scaling the $D$ dependent constant term with $\Delta t<1$, thereby controlling the growth rate.

\paragraph{Improved Backward Stability.}
With the modification, the backpropagated gradient is given by~\eqref{eq:backprop_grad}, where the matrix $\bfJ_{i:D}$ is given by
\begin{equation}
\label{eq:gradient_product_improved}
    \bfJ_{i:D} =  \prod_{j=i+1}^D \left( \bfI + \Delta t \cdot \nabla_{\bfX_{j-1}} f(\bfX_{j-1}; \bftheta_{j-1})\right) ,
\end{equation}
where $f \in \{ f^{\rm Pre}, f^{\rm Peri} \}$. The factor $\Delta t < 1$ scales the local sensitivity of each block, mitigating the potential for gradient explosion and thus improving training stability for both Pre-LN and Peri-LN Transformers.

\paragraph{Overall Stabilization.} Overall, applying $\Delta t<1$ provides a simple yet practical mechanism to control both forward and backward signal propagation, complementing the inherent stabilizing effect of Peri-LN, \emph{at no additional computational or memory cost}. We emphasize again that these two forms of stabilization correspond to distinct aspects: forward stability determines the hidden-state growth of the trained Transformer, while backward stability impacts the stability of the training dynamics.

\section{Experimental Results}
\label{sec:experiments}
In this section, we present experimental results that support the theoretical findings of this work.

We experiment with the GPT-2 family of architectures using the implementation provided in~\cite{Karpathy2022}. We consider three models, GPT-2, GPT-2 Large, and GPT-2 XL, with sizes ranging from $100$M to $1.5$B parameters. We perform pretraining on OpenWebText dataset, which was originally curated in~\cite{Gokaslan2019OpenWeb} 
and includes approximately 9 billion training tokens and 4 million validation tokens.

To isolate the effect of layer normalization and residual scaling, we use hyperparameters (weight decay, learning rate, etc.) tuned for Pre-LN models and do not optimize them for Peri-LN. Even without hyperparameter tuning, Peri-LN achieves competitive performance, demonstrating our theoretical claims are robust and not reliant on cherry-picked results.

All experiments are conducted using NVIDIA H200 GPUs. The code and trained models will be made publicly available upon publication.

\paragraph{Training Stability Test}
We compare the numerical stability of Pre-LN and Peri-LN models under different weight decay settings by counting the number of diverged runs over 5 trials on the base GPT-2 model. We train for $20$K iterations, which is sufficient to demonstrate significant differences. 
As shown in~\Cref{tab:blowup_results}, Pre-LN training can diverge even when all hyperparameters, including weight decay, are properly selected.  In addition, weight decay alone is not sufficient to guarantee stability. In contrast, Peri-LN models remain stable in all trials, highlighting their robust training regardless of weight decay.

\begin{table}[htbp]
\centering
\caption{Diverged run count across 5 trials. The performance of the converged runs only is reported in~\Cref{tab:gpt2_main}
}
\begin{tabular}{@{}lcc@{}}
\toprule
LN Setting & Weight Decay On & Weight Decay Off \\
\midrule
Pre-LN   & $1$ out of $5$ & $3$ out of $5$ \\
Peri-LN  & $\bf 0$ \textbf{out of} $\bf 5$ & $\bf 0$ \textbf{out of} $\bf 5$ \\
LN Off   & $5$ out of $5$ & ---   \\
\bottomrule
\end{tabular}
\label{tab:blowup_results}
\end{table}

\begin{table*}[t]
\centering
\caption{Model performance comparison for different choices of GPT-2 variants, LN types and $\Delta t$ scaling. \\
Diverged runs for Pre-LN are not reported. Even under sub-optimal hyperparameters, Peri-LN achieve on-par performance with Pre-LN while being significantly more stable; see also~\Cref{tab:blowup_results}}
\renewcommand{\arraystretch}{1.14}
\small
\begin{tabular}{lcccccccccc}
\toprule
Size & LN type & $\Delta t$ & Val Loss & Perplexity & Rouge1 & Rouge2 & RougeL & BertP & BertR & BertF1 \\
\midrule
\multirow{4}{*}{$124$M} 
  & Pre-LN  & $1$   & $5.43$ & $247.52$ & $37.50\%$ & $9.62\%$ & $23.94\%$ & $87.40\%$ & $85.38\%$ & $86.38\%$ \\
  & Pre-LN  & $0.1$ & $3.13$ & $24.43$  & $62.45\%$ & $25.48\%$ & $42.70\%$ & $90.27\%$ & $89.75\%$ & $90.00\%$ \\
  & Peri-LN & $1$   & $3.12$ & $24.17$  & $62.26\%$ & $25.42\%$ & $42.80\%$ & $90.25\%$ & $89.72\%$ & $89.99\%$ \\
  & Peri-LN & $0.1$ & $\bf 3.10$ & $\bf 23.63$ & $\bf 62.59\%$ & $\bf 25.71\%$ & $\bf 43.01\%$ & $\bf 90.28\%$ & $\bf 89.76\%$ & $\bf 90.02\%$ \\
\midrule
\multirow{4}{*}{$774$M}  
  & Pre-LN  & $1$   & $2.90$ & $19.44$ & $63.83\%$ & $27.20\%$ & $44.95\%$ & $90.44\%$ & $89.98\%$ & $90.21\%$ \\
  & Pre-LN  & $0.1$ & $2.91$ & $19.61$ & $63.80\%$ & $27.17\%$ & $44.93\%$ & $90.43\%$ & $90.00\%$ & $90.20\%$ \\
  & Peri-LN & $1$   & $2.91$ & $19.70$ & $63.63\%$ & $27.00\%$ & $44.88\%$ & $90.44\%$ & $89.98\%$ & $90.21\%$ \\
  & Peri-LN & $0.1$ & $\bf 2.89$ & $\bf 19.24$ & $\bf 63.89\%$ & $\bf 27.29\%$ & $\bf 45.09\%$ & $\bf 90.47\%$ & $\bf 90.02\%$ & $\bf 90.24\%$ \\
\midrule
\multirow{4}{*}{$1.5$B}
  & Pre-LN  & $1$   & $2.88$ & $19.14$ & $64.13\%$ & $27.46\%$ & $45.19\%$ & $90.49\%$ & $90.07\%$ & $90.28\%$ \\
  & Pre-LN  & $0.1$ & $2.88$ & $19.14$ & $64.14\%$ & $27.49\%$ & $45.19\%$ & $90.49\%$ & $90.06\%$ & $90.27\%$ \\
  & Peri-LN & $1$   & $2.89$ & $19.36$ & $64.04\%$ & $27.35\%$ & $45.09\%$ & $90.50\%$ & $90.08\%$ & $90.29\%$ \\
  & Peri-LN & $0.1$ & $\bf 2.87$ & $\bf 18.93$ & $\bf 64.21\%$ & $\bf 27.58\%$ & $\bf 45.30\%$ & $\bf 90.52\%$ & $\bf 90.09\%$ & $\bf 90.30\%$ \\
\bottomrule
\end{tabular}
\label{tab:gpt2_main}
\end{table*}

\begin{figure*}[t]  
    \centering
    \includegraphics[width=0.98\textwidth]{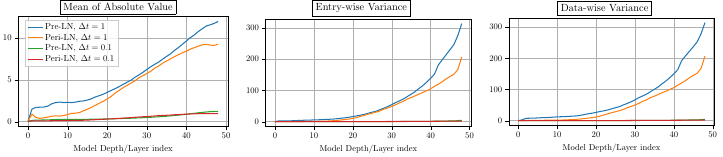}
        \caption{Moments of hidden states across layers for the trained GPT-2 XL. With tuned weight decay for Pre-LN, the growth rate remains below the theoretical exponential upper bound~\eqref{eq:expo_growth_discrete}; exponential growth is observed in, e.g., \cite{kim2025peri}. The residual step scaling (see \Cref{sec:delta_t})  effectively controls the growth at no extra cost.}
    \label{fig:gpt2_xl_growth}
\end{figure*}

We also conduct experiments on the GPT-2 variants under difference layer normalization strategies and residual step scalings. For the base GPT-2 model, we use $100$k iterations, and for the GPT-2 Large and XL variants, we use $60$k iterations. We evaluate the performance using several widely adopted LLM metrics and report the results in~\Cref{tab:gpt2_main}. 

From the top row of~\Cref{tab:gpt2_main}, we see again that Pre-LN can suffer from training instability, leading to poor performance even when all hyperparameters are tuned. While Peri-LN does not consistently improve performance, it does not degrade performance relative to Pre-LN. These results justify Peri-LN's adoption, particularly given that the experiments use hyperparameters chosen for Pre-LN and that Peri-LN provides significantly enhanced training stability; see~\Cref{tab:blowup_results}.

\paragraph{Benefits of Scaled Residual Steps.}
We highlight the performance and stability improvements achieved through using the residual steps scaling, validating our analysis in~\Cref{sec:delta_t}. 

First, we see that combining Peri-LN with residual scaling $\Delta t = 0.1$ consistently yields the best performance across all metrics, corroborating our theoretical analysis on improved generalization bounds~\eqref{eq:forward_stable_improved}.

Second, from~\Cref{tab:gpt2_main}, the Pre-LN model with $\Delta t=1$ suffers from training instability and thus degraded performance. By reducing the residual scaling to $\Delta t = 0.1$, the model becomes stable and achieves much better performance, demonstrating its effectiveness. This aligns with the improved gradient sensitivity~\eqref{eq:gradient_product_improved}.

Third, from~\Cref{fig:gpt2_xl_growth}, we see that the residual scaling effectively controls the growth of hidden states, in terms of both their mean absolute value and variance, across layers. In particular, when $\Delta t = 0.1$, the hidden states grow at a substantially lower rate, resulting in more regular hidden state dynamics while maintaining model performance.
The observation is consistent with our analysis in~\Cref{sec:delta_t}.

Finally, we emphasize again that the residual step scaling, which requires only a simple code modification, provides the above three benefits \emph{without any additional computational or memory cost,} making it a compelling technique in practice. This practical effectiveness further highlights the value of the continuous-time perspective, which not only yields theoretical insights but also guides useful architectural modifications.

\section{Discussion}
In this paper, we presented a novel theoretical framework grounded in optimal control theory to study the effects of layer normalization placements on Transformer stability. Our theory explains the empirical observations reported in the literature, which have previously lacked sufficient theoretical understanding. Building on these insights, we introduce a residual step scaling that enhances the stability and performance of Transformers. Moreover, our experiments show that even under tuned hyperparameters, Pre-LN Transformers can still exhibit training instability. In contrast, Peri-LN models achieve comparable performance while remaining stable throughout training, even when the hyperparameters are suboptimal. 

Looking forward, our framework provides a principled workflow to determine whether new architectural modifications lead to regular trained models. In particular, it serves as theoretical criteria for screening architectures before expensive empirical training, thereby guiding the design of future Transformers.

\ifCLASSOPTIONcaptionsoff
  \newpage
\fi

\bibliographystyle{IEEEtran}
\bibliography{iclr2025_conference}

@article{kim2025peri,
  title={Peri-LN: Revisiting Layer Normalization in the Transformer Architecture},
  author={Kim, Jeonghoon and Lee, Byeongchan and Park, Cheonbok and Oh, Yeontaek and Kim, Beomjun and Yoo, Taehwan and Shin, Seongjin and Han, Dongyoon and Shin, Jinwoo and Yoo, Kang Min},
  journal={arXiv preprint arXiv:2502.02732},
  year={2025}
}

@misc{ba2016layernormalization,
      title={Layer Normalization}, 
      author={Jimmy Lei Ba and Jamie Ryan Kiros and Geoffrey E. Hinton},
      year={2016},
      eprint={1607.06450},
      archivePrefix={arXiv},
      primaryClass={stat.ML},
      url={https://arxiv.org/abs/1607.06450}, 
}

@article{FournierGuillin2015,
  author = {Fournier, Nicolas and Guillin, Arnaud},
  title = {On the rate of convergence in {W}asserstein distance of the empirical measure},
  journal = {Probability Theory and Related Fields},
  volume = {162},
  number = {3-4},
  year = {2015},
  pages = {707--738},
  doi = {10.1007/s00440-014-0583-7},
  publisher = {Springer}
}

@inproceedings{lu2020understanding,
  title={Understanding and Improving Transformer From a Multi-Particle Dynamic System Point of View.},
  author={Lu, Yiping and Li, Zhuohan and He, Di and Sun, Zhiqing and Dong, Bin and Qin, Tao and Wang, Liwei and Liu, Tie-yan},
  booktitle={ICLR 2020 Workshop on Integration of Deep Neural Models and Differential Equations},
  year={2020},
}

@inproceedings{xiong2020layer,
  title={On layer normalization in the transformer architecture},
  author={Xiong, Ruibin and Yang, Yunchang and He, Di and Zheng, Kai and Zheng, Shuxin and Xing, Chen and Zhang, Huishuai and Lan, Yanyan and Wang, Liwei and Liu, Tieyan},
  booktitle={International Conference on Machine Learning},
  pages={10524--10533},
  year={2020},
  organization={PMLR}
}

@inproceedings{he2016deep,
  title={Deep residual learning for image recognition},
  author={He, Kaiming and Zhang, Xiangyu and Ren, Shaoqing and Sun, Jian},
  booktitle={Proceedings of the IEEE conference on computer vision and pattern recognition},
  pages={770--778},
  year={2016}
}

@article{haber2017stable,
  title={Stable architectures for deep neural networks},
  author={Haber, Eldad and Ruthotto, Lars},
  journal={Inverse problems},
  volume={34},
  number={1},
  pages={014004},
  year={2017},
  publisher={IOP Publishing}
}

@article{ruthotto2020deep,
  title={Deep neural networks motivated by partial differential equations},
  author={Ruthotto, Lars and Haber, Eldad},
  journal={Journal of Mathematical Imaging and Vision},
  volume={62},
  number={3},
  pages={352--364},
  year={2020},
  publisher={Springer}
}

@article{zhang2023mean,
  title={A mean-field games laboratory for generative modeling},
  author={Zhang, Benjamin J and Katsoulakis, Markos A},
  journal={arXiv preprint arXiv:2304.13534},
  year={2023}
}

@misc{Karpathy2022,
  author = {Andrej Karpathy},
  title = {NanoGPT},
  year = {2022},
  publisher = {GitHub},
  journal = {GitHub repository},
  howpublished = {\url{https://github.com/karpathy/nanoGPT}},
  commit = {325be85d9be8c81b436728a420e85796c57dba7e}
}

@misc{Gokaslan2019OpenWeb,  
	title={{OpenWebText Corpus}},
	author={Aaron Gokaslan and Vanya Cohen},
	howpublished={\url{http://Skylion007.github.io/OpenWebTextCorpus}}, 
	year={2019}
}

@book{bensoussan2013mean,
  title={Mean field games and mean field type control theory},
  author={Bensoussan, Alain and Frehse, Jens and Yam, Phillip and others},
  volume={101},
  year={2013},
  publisher={Springer}
}

@article{child2019generating,
  title={Generating long sequences with sparse transformers},
  author={Child, Rewon and Gray, Scott and Radford, Alec and Sutskever, Ilya},
  journal={arXiv preprint arXiv:1904.10509},
  year={2019}
}

@article{Kan2024LSEMINK,
  author  = {Kelvin Kan and James G. Nagy and Lars Ruthotto},
  title   = {{LSEMINK}: a modified {N}ewton--{K}rylov method for Log-Sum-Exp minimization},
  journal = {Electron. Trans. Numer. Anal.},
  volume  = {60},
  year    = {2024},
  pages   = {618--635},
  doi     = {10.1553/etna_vol60s618},
}

@article{gao2017properties,
  title={On the properties of the softmax function with application in game theory and reinforcement learning},
  author={Gao, Bolin and Pavel, Lacra},
  journal={arXiv preprint arXiv:1704.00805},
  year={2017}
}

@article{zhang2019root,
  title={Root mean square layer normalization},
  author={Zhang, Biao and Sennrich, Rico},
  journal={Advances in Neural Information Processing Systems},
  volume={32},
  year={2019}
}

@article{kedia2024transformers,
  title={Transformers get stable: an end-to-end signal propagation theory for language models},
  author={Kedia, Akhil and Zaidi, Mohd Abbas and Khyalia, Sushil and Jung, Jungho and Goka, Harshith and Lee, Haejun},
  journal={arXiv preprint arXiv:2403.09635},
  year={2024}
}

@inproceedings{
sun2024massive,
title={Massive Activations in Large Language Models},
author={Mingjie Sun and Xinlei Chen and J Zico Kolter and Zhuang Liu},
booktitle={First Conference on Language Modeling},
year={2024},
url={https://openreview.net/forum?id=F7aAhfitX6}
}

@article{OLMo2,
  publtype={informal},
  author={Team OLMo and Pete Walsh and Luca Soldaini and Dirk Groeneveld and Kyle Lo and Shane Arora and Akshita Bhagia and Yuling Gu and Shengyi Huang and Matt Jordan and Nathan Lambert and Dustin Schwenk and Oyvind Tafjord and Taira Anderson and David Atkinson and Faeze Brahman and Christopher Clark and Pradeep Dasigi and Nouha Dziri and Michal Guerquin and Hamish Ivison and Pang Wei Koh and Jiacheng Liu and Saumya Malik and William Merrill and Lester James V. Miranda and Jacob Morrison and Tyler Murray and Crystal Nam and Valentina Pyatkin and Aman Rangapur and Michael Schmitz and Sam Skjonsberg and David Wadden and Christopher Wilhelm and Michael Wilson and Luke Zettlemoyer and Ali Farhadi and Noah A. Smith and Hannaneh Hajishirzi},
  title={2 OLMo 2 Furious},
  year={2025},
  month={January},
  cdate={1735689600000},
  journal={CoRR},
  volume={abs/2501.00656},
  url={https://doi.org/10.48550/arXiv.2501.00656}
}

@article{Gemma2,
  publtype={informal},
  author={Morgane Rivière and Shreya Pathak and Pier Giuseppe Sessa and Cassidy Hardin and Surya Bhupatiraju and Léonard Hussenot and Thomas Mesnard and Bobak Shahriari and Alexandre Ramé and Johan Ferret and Peter Liu and Pouya Tafti and Abe Friesen and Michelle Casbon and Sabela Ramos and Ravin Kumar and Charline Le Lan and Sammy Jerome and Anton Tsitsulin and Nino Vieillard and Piotr Stanczyk and Sertan Girgin and Nikola Momchev and Matt Hoffman and Shantanu Thakoor and Jean-Bastien Grill and Behnam Neyshabur and Olivier Bachem and Alanna Walton and Aliaksei Severyn and Alicia Parrish and Aliya Ahmad and Allen Hutchison and Alvin Abdagic and Amanda Carl and Amy Shen and Andy Brock and Andy Coenen and Anthony Laforge and Antonia Paterson and Ben Bastian and Bilal Piot and Bo Wu and Brandon Royal and Charlie Chen and Chintu Kumar and Chris Perry and Chris Welty and Christopher A. Choquette-Choo and Danila Sinopalnikov and David Weinberger and Dimple Vijaykumar and Dominika Rogozinska and Dustin Herbison and Elisa Bandy and Emma Wang and Eric Noland and Erica Moreira and Evan Senter and Evgenii Eltyshev and Francesco Visin and Gabriel Rasskin and Gary Wei and Glenn Cameron and Gus Martins and Hadi Hashemi and Hanna Klimczak-Plucinska and Harleen Batra and Harsh Dhand and Ivan Nardini and Jacinda Mein and Jack Zhou and James Svensson and Jeff Stanway and Jetha Chan and Jin Peng Zhou and Joana Carrasqueira and Joana Iljazi and Jocelyn Becker and Joe Fernandez and Joost van Amersfoort and Josh Gordon and Josh Lipschultz and Josh Newlan and Ju-yeong Ji and Kareem Mohamed and Kartikeya Badola and Kat Black and Katie Millican and Keelin McDonell and Kelvin Nguyen and Kiranbir Sodhia and Kish Greene and Lars Lowe Sjösund and Lauren Usui and Laurent Sifre and Lena Heuermann and Leticia Lago and Lilly McNealus},
  title={Gemma 2: Improving Open Language Models at a Practical Size},
  year={2024},
  cdate={1704067200000},
  journal={CoRR},
  volume={abs/2408.00118},
  url={https://doi.org/10.48550/arXiv.2408.00118}
}

@article{team2025gemma,
  title={Gemma 3 technical report},
  author={Team, Gemma and Kamath, Aishwarya and Ferret, Johan and Pathak, Shreya and Vieillard, Nino and Merhej, Ramona and Perrin, Sarah and Matejovicova, Tatiana and Ram{\'e}, Alexandre and Rivi{\`e}re, Morgane and others},
  journal={arXiv preprint arXiv:2503.19786},
  year={2025}
}

@book{nocedal2006numerical,
  title={Numerical optimization},
  author={Nocedal, Jorge and Wright, Stephen J},
  year={2006},
  publisher={Springer}
}

@article{kan2025optimal,
  title={Optimal Control for Transformer Architectures: Enhancing Generalization, Robustness and Efficiency},
  author={Kan, Kelvin and Li, Xingjian and Zhang, Benjamin J and Sahai, Tuhin and Osher, Stanley and Katsoulakis, Markos A},
  journal={arXiv preprint arXiv:2505.13499},
  year={2025}
}

@inproceedings{Wang2019LearningDT,
  title={Learning Deep Transformer Models for Machine Translation},
  author={Qiang Wang and Bei Li and Tong Xiao and Jingbo Zhu and Changliang Li and Derek F. Wong and Lidia S. Chao},
  booktitle={Annual Meeting of the Association for Computational Linguistics},
  year={2019},
  url={https://api.semanticscholar.org/CorpusID:174799399}
}

@article{liu2019variance,
  title={On the variance of the adaptive learning rate and beyond},
  author={Liu, Liyuan and Jiang, Haoming and He, Pengcheng and Chen, Weizhu and Liu, Xiaodong and Gao, Jianfeng and Han, Jiawei},
  journal={arXiv preprint arXiv:1908.03265},
  year={2019}
}

@article{Popel2018TrainingTF,
  title={Training Tips for the Transformer Model},
  author={Martin Popel and Ondřej Bojar},
  journal={The Prague Bulletin of Mathematical Linguistics},
  volume={110},
  pages={43--70},
  year={2018},
  doi={10.2478/pralin-2018-0002},
  url={https://arxiv.org/abs/1804.00247}
}

@inproceedings{
baevski2018adaptive,
title={Adaptive Input Representations for Neural Language Modeling},
author={Alexei Baevski and Michael Auli},
booktitle={International Conference on Learning Representations},
year={2019},
url={https://openreview.net/forum?id=ByxZX20qFQ},
}

@inproceedings{wang-etal-2019-learning-deep,
    title = "Learning Deep Transformer Models for Machine Translation",
    author = "Wang, Qiang  and
      Li, Bei  and
      Xiao, Tong  and
      Zhu, Jingbo  and
      Li, Changliang  and
      Wong, Derek F.  and
      Chao, Lidia S.",
    editor = "Korhonen, Anna  and
      Traum, David  and
      M{\`a}rquez, Llu{\'i}s",
    booktitle = "Proceedings of the 57th Annual Meeting of the Association for Computational Linguistics",
    month = jul,
    year = "2019",
    address = "Florence, Italy",
    publisher = "Association for Computational Linguistics",
    url = "https://aclanthology.org/P19-1176/",
    doi = "10.18653/v1/P19-1176",
    pages = "1810--1822"
}

@inproceedings{nguyen-salazar-2019-transformers,
    title = "Transformers without Tears: Improving the Normalization of Self-Attention",
    author = "Nguyen, Toan Q.  and
      Salazar, Julian",
    editor = {Niehues, Jan  and
      Cattoni, Rolando  and
      St{\"u}ker, Sebastian  and
      Negri, Matteo  and
      Turchi, Marco  and
      Ha, Thanh-Le  and
      Salesky, Elizabeth  and
      Sanabria, Ramon  and
      Barrault, Loic  and
      Specia, Lucia  and
      Federico, Marcello},
    booktitle = "Proceedings of the 16th International Conference on Spoken Language Translation",
    month = nov # " 2-3",
    year = "2019",
    address = "Hong Kong",
    publisher = "Association for Computational Linguistics",
    url = "https://aclanthology.org/2019.iwslt-1.17/"
}

@inproceedings{
dettmers2022gptint,
title={{GPT}3.int8(): 8-bit Matrix Multiplication for Transformers at Scale},
author={Tim Dettmers and Mike Lewis and Younes Belkada and Luke Zettlemoyer},
booktitle={Advances in Neural Information Processing Systems},
editor={Alice H. Oh and Alekh Agarwal and Danielle Belgrave and Kyunghyun Cho},
year={2022},
url={https://openreview.net/forum?id=dXiGWqBoxaD}
}

@article{yu2024super,
  title={The super weight in large language models},
  author={Yu, Mengxia and Wang, De and Shan, Qi and Reed, Colorado J and Wan, Alvin},
  journal={arXiv preprint arXiv:2411.07191},
  year={2024}
}

@inproceedings{
fishman2025scaling,
title={Scaling {FP}8 training to trillion-token {LLM}s},
author={Maxim Fishman and Brian Chmiel and Ron Banner and Daniel Soudry},
booktitle={The Thirteenth International Conference on Learning Representations},
year={2025},
url={https://openreview.net/forum?id=E1EHO0imOb}
}

@article{ding2021cogview,
  title={Cogview: Mastering text-to-image generation via transformers},
  author={Ding, Ming and Yang, Zhuoyi and Hong, Wenyi and Zheng, Wendi and Zhou, Chang and Yin, Da and Lin, Junyang and Zou, Xu and Shao, Zhou and Yang, Hongxia and others},
  journal={Advances in neural information processing systems},
  volume={34},
  pages={19822--19835},
  year={2021}
}

@misc{naver2025periln,
  author = {NAVER},
  title = {Stable training: Preventing divergence with {Peri-LN}},
  year = {2025},
  howpublished = {\url{https://clova.ai/en/tech-blog/stable-training-preventing-divergence-with-peri-ln}},
  note = {Accessed: 2025-09-25}
}

@inproceedings{takase-etal-2023-b2t,
    title = "{B}2{T} Connection: Serving Stability and Performance in Deep Transformers",
    author = "Takase, Sho  and
      Kiyono, Shun  and
      Kobayashi, Sosuke  and
      Suzuki, Jun",
    editor = "Rogers, Anna  and
      Boyd-Graber, Jordan  and
      Okazaki, Naoaki",
    booktitle = "Findings of the Association for Computational Linguistics: ACL 2023",
    month = jul,
    year = "2023",
    address = "Toronto, Canada",
    publisher = "Association for Computational Linguistics",
    url = "https://aclanthology.org/2023.findings-acl.192/",
    doi = "10.18653/v1/2023.findings-acl.192",
    pages = "3078--3095"
}

@article{barles1991convergence,
  title={Convergence of approximation schemes for fully nonlinear second order equations},
  author={Barles, Guy and Souganidis, Panagiotis E},
  journal={Asymptotic analysis},
  volume={4},
  number={3},
  pages={271--283},
  year={1991},
  publisher={SAGE Publications Sage UK: London, England}
}

@book{falcone2014semilagrangian,
  title={Semi-Lagrangian approximation schemes for linear and Hamilton—Jacobi equations},
  author={Falcone, Maurizio and Ferretti, Roberto},
  year={2013},
  publisher={SIAM}
}

\newpage
\onecolumn

\appendices
\section{Related Work}\label{sec:related_work}
Layer normalization~\cite{ba2016layernormalization} has become a standard component of virtually all Transformer architectures, playing a critical role in their stability and performance. The original Transformer employed Post-LN~\cite{Wang2019LearningDT}. It has later been reported that Post-LN requires delicate optimization scheduling~\cite{Popel2018TrainingTF,liu2019variance} to achieve stable training and often yields sub-par performance~\cite{xiong2020layer,kim2025peri}.

Pre-LN~\cite{baevski2018adaptive,child2019generating,wang-etal-2019-learning-deep}, a prominent alternative, eliminates the need for careful optimization scheduling and has been shown to achieve improved performance~\cite{nguyen-salazar-2019-Transformers,xiong2020layer}. 

However, it has been widely observed that Pre-LN Transformers are prone to excessively large activations~\cite{dettmers2022gptint,yu2024super,sun2024massive,fishman2025scaling,kim2025peri}, especially in deep and large models. These activations can deteriorate the quality of the learned representation~\cite{kim2025peri}. While existing studies are mostly empirical, or focus on models at initialization~\cite{xiong2020layer,kedia2024Transformers}, our theory provides a principled explanation of why large activations arise for trained Pre-LN models.

To the best of our knowledge, Peri-LN was first used in~\cite{ding2021cogview}, where it was used as an ad-hoc method to stabilize training. Recently, Peri-LN has been deployed in major open-source packages including Olmo2~\cite{OLMo2}, Gemma2~\cite{Gemma2}, and Gemma3~\cite{team2025gemma}. But their documentations provide little explanation and performance comparison. In~\cite{kim2025peri}, a systematic empirical study comparing the three layer normalization strategies was conducted, demonstrating the empirical advantages of Peri-LN. However, since Peri-LN is still relatively new in widespread use, there has been little theoretical analysis of it. In this work, we address this gap by performing a theoretical analysis on the stability and performance benefits of Peri-LN.

\section{Diagnostic Workflow Before Training}
\label{sec:workflow}

The mathematical analysis developed in this paper suggests a systematic workflow for architectural stability analysis and training diagnostics of Transformer architecture variants. This procedure applies to new architecture (e.g., changes in normalization placement, residual scalings, or weight decay), and provides theoretical criteria for screening architectures prior to expensive empirical training.

\paragraph{Step 1: Well-posedness via HJB theory.} 
We first cast the training problem for the proposed architecture as a continuous-time mean-field control problem, analogous to~(9). The existence of a Hamiltonian 
of the associated Hamilton–Jacobi–Bellman (HJB) equation provides a necessary certificate that the training is a well-posed control problem. Theorem~2 shows that this condition fails for Pre-LN, whereas Theorem~4 establishes well-posedness for Peri-LN through bounds on the possible ``velocity fields'' in the Transformer architecture.

\paragraph{Step 2: Forward stability analysis.} 
Conditional on well-posedness in Step 1, Theorems~4--5 provide a means to assess forward stability by bounding the growth of hidden states. This step allows one to identify whether the architecture exhibits controlled, linear/quadratic growth, as in Peri-LN, or exponential growth, as in Theorem~3 for Pre-LN. Moreover, the discretization analysis in Section \ref{sec:delta_t} reveals that scaling residual updates with a factor $\Delta t<1$ sharpens the forward stability bounds by reducing amplification across layers, at no additional computational cost. This provides a simple and universal stabilization knob that can be applied before training.

\paragraph{Step 3: Backward stability analysis.} 
Propositions~7--8 enable a local sensitivity analysis for the backward pass, independent of Step~2, by examining the Jacobian structure of each block and its interaction with layer normalization. Assessing how sensitivities scale with activations reveals whether gradients remain bounded during backpropagation. Combined with Step~2, this yields a comprehensive diagnostic: architectures with both large activations and sensitivity growth are especially prone to gradient explosion. The same $\Delta t < 1$ scaling effectively reduces local sensitivities in each block, mitigating gradient growth and enhancing training stability.

\paragraph{Step 4: Uncertainty quantification.} 
Theorem~6 provides a quantitative Wasserstein bound: input uncertainty from new data, noise, or adversarial perturbations leads to a controlled level of uncertainty in the terminal representation. This worst-case estimate highlights the robustness of Peri-LN, whereas Pre-LN may amplify input uncertainty and produce unbounded differences (Theorem~2). Because this analysis depends only on the forward dynamics (Step~2), it applies to any proposed architecture to assess in- and out-of-distribution stability before training.

This four-step workflow can be applied to candidate architectures before pretraining. It offers a mathematically grounded diagnostics to identify and discard ill-posed or unstable designs, complementing empirical architecture search and reducing the need for costly simulations. 

\section{Properties of Layer Normalization}
We prove properties for layer normalization operations which will be used in our derivations later.

\paragraph{Layer Normalization} We first recall the definition of the layer normalization operation introduced in~\Cref{sec:OC_formulation}. Given a hidden state $\bfX \in \mathbb{R}^{d \times n}$, layer normalization applies to each of its tokens (columns) $\bfx \in \mathbb{R}^d$ as follows
\begin{equation}\label{eq:LN_appendix}
    {\rm LN}(\bfx; \bfgamma, \bfbeta) = \bfgamma \odot \hat{\bfx} + \bfbeta,
\end{equation}
where $\bfgamma, \bfbeta \in \mathbb{R}^d$ are trainable parameters, $\odot$ is the Hadamard element-wise product, $\hat{\bfx}$ is given by
\begin{equation}\label{eq:standardized_appendix}
    \hat{\bfx} = \frac{\bfx - \mu}{\sigma}, \quad \text{with} \quad \mu = \frac{1}{d} \sum_{l=1}^d x_l, \quad \sigma = \sqrt{\frac{1}{d} \sum_{l=1}^d (x_l - \mu)^2}.
\end{equation}

In the following proposition, we show that the output of the layer normalization always lies on an ellipsoid. We first denote $\bfGamma := \textnormal{diag}(\gamma_1, \gamma_2, ..., \gamma_d) \in \mathbb{R}^{d\times d}$, where $\gamma_i$ is the $i$th entry of $\bfgamma$.

\settheoremnumber{lemma:ellipsoid}
\begin{lemma}
    The LayerNorm output $\bfz=\text{LayerNorm}(\bfx; \bfgamma, \bfbeta)$ always lies on the ellipsoid 
    $$
    \mathcal{E} = \left\{ \mathbf{z} \in \mathbb{R}^d : (\mathbf{z} - \bfbeta)^\top \bfGamma^{-2} (\mathbf{z} - \bfbeta) = d \right\},
    $$
    where
    \begin{equation}\label{eq:Sigma_diag}
    \bfGamma^{-2} = \textnormal{diag}(\gamma_1^{-2}, \gamma_2^{-2}, ..., \gamma_d^{-2}) \in \mathbb{R}^{d\times d}.
    \end{equation} 
\end{lemma}
\restoretheoremnumber

\begin{proof}
        Let $\bfz=\text{LN}(\bfx; \bfgamma, \bfbeta)$ for $\bfx \in \mathbb{R}^d$. We have
\begin{align*}
    (\bfz - \bfbeta)^\top \bfGamma^{-2} (\bfz - \bfbeta) &= (\bfgamma \odot \hat{\bfx})^\top \bfGamma^{-2} (\bfgamma \odot \hat{\bfx}), \quad &\text{by~(\ref{eq:LN_appendix})}, \\
    &= \hat{\bfx}^\top \hat{\bfx}, \quad &\text{by~(\ref{eq:Sigma_diag})}, \\
    &= \left( \frac{\bfx - \mu}{\sigma}\right)^\top \left( \frac{\bfx - \mu}{\sigma}\right), \quad &\text{by definition of $\hat{\bfx}$ in (\ref{eq:standardized_appendix})}, \\
    &= \frac{1}{\sigma^2} \sum_{l=1}^d (x_l - \mu)^2 \\
    &= d, \quad &\text{by definition of $\sigma$ in (\ref{eq:standardized_appendix})}.
\end{align*}
Thus, $\bfz \in \mathcal{E}$.
\end{proof}

\begin{lemma}\label{prop:grad_LN}
    The gradient of $\rm LN$ with respect to $\bfx$ is given by 
    $$
    \nabla {\rm LN}(\bfx; \bfgamma, \bfbeta) = \frac{{\rm diag}(\bfgamma)}{\sigma} - \frac{1}{d} \frac{(\bfx-\mu)(\bfgamma \odot (\bfx-\mu))^\top}{\sigma^3}. 
    $$
    Moreover, for any $c > 0$, 
    $$
    \nabla {\rm LN}(c\bfx; \bfgamma, \bfbeta) = \frac{1}{c}\nabla {\rm LN}(\bfx; \bfgamma, \bfbeta).
    $$
\end{lemma}

\begin{proof}
    Consider the $i$th entry of ${\rm LN}(\bfx)$, which is given by
    $$
    [{\rm LN}(\bfx)]_i = \frac{\gamma_i (x_i-\mu)}{\sigma} + \beta_i = \frac{\gamma_i (x_i-\mu)}{\sqrt{\frac{1}{d} \sum_{j=1}^d (x_j - \mu)^2}} + \beta_i.
    $$
    Its gradient is given by
    $$
    \nabla [{\rm LN}(\bfx)]_i = \frac{\gamma_i \bfe_i }{\sqrt{\frac{1}{d} \sum_{j=1}^d (x_j - \mu)^2}} - \frac{1}{d} \frac{\gamma_i(x_i-\mu)(\bfx-\mu)}{(\frac{1}{d} \sum_{j=1}^d (x_j - \mu)^2)^{\frac{3}{2}}}.
    $$
    Thus, we have 
    \begin{align*}
    \nabla {\rm LN}(\bfx) &= \frac{{\rm diag}(\bfgamma)}{\sqrt{\frac{1}{d} \sum_{j=1}^d (x_j - \mu)^2}} - \frac{1}{d} \frac{(\bfx-\mu)(\bfgamma \odot (\bfx-\mu))^\top}{(\frac{1}{d} \sum_{j=1}^d (x_j - \mu)^2)^{\frac{3}{2}}} \\
    &= \frac{{\rm diag}(\bfgamma)}{\sigma} - \frac{1}{d} \frac{(\bfx-\mu)(\bfgamma \odot (\bfx-\mu))^\top}{\sigma^3}.
    \end{align*}
    For any $c > 0$, we can verify
    \begin{align*}
    \nabla {\rm LN}(c\bfx) &= \frac{{\rm diag}(\bfgamma)}{c \sigma} - \frac{1}{d} \frac{(c\bfx-\mu)(\bfgamma \odot (c\bfx-\mu))^\top}{c^3 \sigma^3} \\
    &= \frac{{\rm diag}(\bfgamma)}{c \sigma} - \frac{1}{d} \frac{(\bfx-\mu)(\bfgamma \odot (\bfx-\mu))^\top}{c \sigma^3} \\
    &= \frac{1}{c} \nabla {\rm LN}(\bfx).
    \end{align*}
\end{proof}

\paragraph{RMSNorm} Another common layer normalization operation is Root Mean Squared Layer Normalization (RMSNorm)~\cite{zhang2019root}. It rescales a given data by root mean square (RMS). It reads
\begin{equation}\label{eq:RMSNorm}
    \text{RMSNorm}(\bfx; \bfgamma) = \bfgamma \odot
    \tilde{\bfx},
\end{equation}
where $\bfgamma\in \mathbb{R}^d$ is trainable parameters,
\begin{equation}\label{eq:standardized_RMS}
    \tilde{\bfx} = \frac{\bfx}{\text{RMS}(\bfx)} = \frac{\bfx}{\sqrt{\frac{1}{d} \| \bfx \|_2^2}} = \frac{\bfx}{\sqrt{\frac{1}{d} \sum_{i=1}^d x_i^2}}.
\end{equation}
The main difference between RMSNorm and LayerNorm is that RMSNorm is \emph{not re-centering invariant}. In other words,
\begin{enumerate}
    \item RMSNorm has no mean subtraction, and
    \item RMSNorm has no learnable bias $\bfbeta$.
\end{enumerate}
By skipping the mean subtraction step, RMSNorm offers slightly improved computational efficiency. This can be more preferable in large-scale experiments where efficiency is a priority.

In the following proposition, we show that the output of RMSNorm lies on an ellipsoid centered at the origin and defined only by the trainable parameter $\bfgamma$. 

\begin{lemma}
    The RMSNorm output $\bfz = \textnormal{RMSNorm}(\bfx; \bfgamma)$  always lies on the ellipsoid
    $$
    \mathcal{E}' = \left\{ \mathbf{z} \in \mathbb{R}^d : \mathbf{z}^\top \bfGamma^{-2} \mathbf{z} = d \right\},
    $$
    where $d$ is the hidden state dimension, 
    \begin{equation}\label{eq:Sigma_diag_RMS}
    \bfGamma^{-2} = \textnormal{diag}(\gamma_1^{-2}, \gamma_2^{-2}, ..., \gamma_d^{-2}) \in \mathbb{R}^{d\times d},
    \end{equation} 
    and $\gamma_i$ is the $i$th entry of $\bfgamma$.
\end{lemma}

\begin{proof}
    Consider an output $\bfz=\text{RMSNorm}(\bfx; \bfgamma)$ for $\bfx \in \mathbb{R}^d$, with entries not all zero (otherwise we will have division by zero),
    \begin{align*}
        \bfz^\top \bfGamma^{-2} \bfz &= (\bfgamma \odot
    \tilde{\bfx})^\top \bfGamma^{-2} (\bfgamma \odot
    \tilde{\bfx}), \quad &\text{by (\ref{eq:RMSNorm}),}\\
    &= \tilde{\bfx}^\top \tilde{\bfx}, \quad &\text{by (\ref{eq:Sigma_diag_RMS}),}\\
    &= \frac{1}{\frac{1}{d}\sum_{i=1}^d x_i^2} \sum_{i=1}^d x_i^2,  \quad &\text{by (\ref{eq:standardized_RMS}),} \\
    &= d.
    \end{align*}
    Thus, we have that $\bfz \in \mathcal{E}'$.
\end{proof} 

\section{Gradient of Multihead Self-Attention Module}
Recall that the self-attention module is given by
\begin{equation}
    f_{\rm attn}(\bfX) = \sum_{h=1}^H  \bfW^h \bfV^h \bfX\, {\rm softmax} \left( \frac{(\bfK^h \bfX)^\top \bfQ^h \bfX}{\sqrt{k}} \right)
\end{equation}

\begin{proposition}\label{prop:grad_attn}
    Denote the $i$th column of $\bfX$ to be $\bfx_i$, and the $j$th column of $f_{\rm attn}(\bfX)$ to be $[f_{\rm attn}(\bfX)]_j$, then the gradient $\nabla_{\bfx_i} [f_{\rm attn}(\bfX)]_j \in \mathbb{R}^{d \times d}$ is given by
    \begin{samepage}
    \begin{align}\label{eq:grad_attn}
    & \nabla_{\bfx_i} [f_{\rm attn}(\bfX)]_j = \\
    & \sum_{h=1}^H \left( a^h_i + \frac{1}{\sqrt{k}} \left((\bfQ^h)^\top \bfK^h \bfX \mathbf{1}_{i=j} + (\bfK^h)^\top \bfQ^h \bfx_j \bfe_i^\top \right) \left({\rm diag}(\bfa^h)-\bfa^h (\bfa^h)^\top \right) \bfX^\top \right) (\bfW^h \bfV^h)^\top ,
    \end{align}
    \end{samepage}
    where $\bfe_i \in \mathbb{R}^{n}$ is the $i$th standard basis vector, 
    $$
    \bfa^h = {\rm softmax} \left( \frac{(\bfK^h \bfX)^\top \bfQ^h \bfx_j}{\sqrt{k}} \right) \in \mathbb{R}^{n},
    $$
    $a^h_i \in \mathbb{R}$ is the $i$th entry of $\bfa^h$, and $\mathbf{1}_{i=j}$ equals $1$ if $i=j$ and 0 otherwise. Moreover, from (\ref{eq:grad_attn}), the gradient $\nabla_{\bfx_i} [f_{\rm attn}(\bfX)]_j \in \mathbb{R}^{d \times d}$ depends linearly on the matrices $\bfV^h$ and $\bfW^h$, respectively.
\end{proposition}

\begin{proof}
    We denote 
    $$
    \bfs^h =  \frac{(\bfK^h \bfX)^\top \bfQ^h \bfx_j}{\sqrt{k}} \in \mathbb{R}^{n},
    $$
    and we have $\bfa^h = {\rm softmax}(\bfs^h)$. We also note that 
    $$
    [f_{\rm attn}(\bfX)]_j = \sum_{h=1}^H \bfW^h \bfV^h \bfX {\rm softmax} \left( \frac{(\bfK^h \bfX)^\top \bfQ^h \bfx_j}{\sqrt{k}} \right) = \sum_{h=1}^H \bfW^h \bfV^h \bfX {\rm softmax} (\bfs^h) = \sum_{h=1}^H \bfW^h \bfV^h \bfX \bfa^h.
    $$
    We have 
    \begin{align}
        \nabla_{\bfx_i} [f_{\rm attn}(\bfX)]_j &= \nabla_{\bfx_i} \left( \sum_{h=1}^H \bfW^h \bfV^h \bfX \bfa^h \right) \\
        &=  \sum_{h=1}^H \nabla_{\bfx_i} \left(   \bfV^h \bfX \bfa^h \right) (\bfW^h)^\top \\
        &=  \sum_{h=1}^H \nabla_{\bfx_i} \left( \sum_{l=1}^n  \bfV^h \bfx_l a^h_l \right) (\bfW^h)^\top, \\
        &= \sum_{h=1}^H   \sum_{l=1}^n \left( \nabla_{\bfx_i} (\bfV^h \bfx_l) a^h_l +  \nabla_{\bfx_i}(a^h_l) \bfx_l^\top (\bfV^h )^\top \right) (\bfW^h)^\top, \\
        & \text{since $\nabla_{\bfx_i} \bfV^h \bfx_l = (\bfV^h)^\top$ if $i=l$ and ${\bf 0}_{d \times k}$ otherwise, we have} \\
        &= \sum_{h=1}^H  \left(  a^h_i (\bfV^h)^\top + \sum_{l=1}^n \nabla_{\bfx_i}(a^h_l) \bfx_l^\top (\bfV^h )^\top \right) (\bfW^h)^\top \\
        &= \sum_{h=1}^H  \left(  a^h_i (\bfV^h)^\top + \nabla_{\bfx_i}(\bfa^h) \bfX^\top (\bfV^h )^\top \right) (\bfW^h)^\top.\label{eq:grad_attn_interm1}
    \end{align}
    Here, 
    \begin{align}
    \nabla_{\bfx_i}(\bfa^h) &= \nabla_{\bfx_i} \bfs^h \cdot \nabla_{\bfs^h} \bfa^h \\
    &= \nabla_{\bfx_i} \bfs^h \cdot \nabla_{\bfs^h} {\rm softmax} (\bfs^h) \\
    &= \nabla_{\bfx_i} \bfs^h \cdot ({\rm diag}(\bfa^h) - \bfa^h (\bfa^h)^\top) \label{eq:grad_attn_interm2},
    \end{align}
by~\cite{gao2017properties}[Proposition~2]. Denote $s^h_l$ as the $l$th entry of $\bfs^h$, we have
\begin{align}
    \nabla_{\bfx_i} s^h_l &= \nabla_{\bfx_i} \left( \frac{1}{\sqrt{k}} (\bfK^h \bfx_l)^\top \bfQ^h \bfx_j \right) \\
    &= \frac{1}{\sqrt{k}} \nabla_{\bfx_i} \left(  (\bfK^h \bfx_l)^\top \bfQ^h \bfx_j \right) \\
    &= \frac{1}{\sqrt{k}}  \left(  (\bfQ^h)^\top (\bfK^h \bfx_l)\mathbf{1}_{i=j} + (\bfK^h)^\top \bfQ^h \bfx_j \mathbf{1}_{i=l} \right).
\end{align}
Thus,
\begin{equation}
    \nabla_{\bfx_i} \bfs^h = \frac{1}{\sqrt{k}}  \left(  (\bfQ^h)^\top (\bfK^h \bfX) \mathbf{1}_{i=j}+ (\bfK^h)^\top \bfQ^h \bfx_j \bfe_i^\top \right). \label{eq:grad_attn_interm3}
\end{equation}
Plugging (\ref{eq:grad_attn_interm3}) into (\ref{eq:grad_attn_interm2}) and that result into (\ref{eq:grad_attn_interm1}) yields
\begin{align}
& \nabla_{\bfx_i} [f_{\rm attn}(\bfX)]_j \\ &=
\sum_{h=1}^H  \left(  a^h_i (\bfV^h)^\top + \frac{1}{\sqrt{k}}  \left(  (\bfQ^h)^\top (\bfK^h \bfX) \mathbf{1}_{i=j}+ (\bfK^h)^\top \bfQ^h \bfx_j \bfe_i^\top \right) ({\rm diag}(\bfa^h) - \bfa^h (\bfa^h)^\top) \bfX^\top (\bfV^h )^\top \right) (\bfW^h)^\top \\
&=  \sum_{h=1}^H \left( a^h_i + \frac{1}{\sqrt{k}} \left((\bfQ^h)^\top \bfK^h \bfX \mathbf{1}_{i=j} + (\bfK^h)^\top \bfQ^h \bfx_j \bfe_i^\top \right) \left({\rm diag}(\bfa^h)-\bfa^h (\bfa^h)^\top \right) \bfX^\top \right) (\bfW^h \bfV^h)^\top.
\end{align}
\end{proof}

\section{Mean Field Control Formulation of Transformer Training}\label{sec:MFC}

\subsection{Mean Field Control Formulation} We recall that the Transformer training problem is given by
\begin{align}\label{eq:discrete_time_OC_obj_appendix}
    & \min_{\bftheta} \; \mathbb{E}_{(\bfX_0, \bfy)} \; G(\bfX_D, \bfy) \\ 
    \begin{split} &\text{s.t.} \; \bfX_{i+\frac{1}{2}} = 
    \bfX_i + \Delta t \cdot f(\bfX_{i}; \bftheta_i) = \bfX_i + \Delta t\cdot f_i(\bfX_{i}),
\end{split}
\end{align}
for $i=0,\frac{1}{2},1,...,D-\frac{1}{2}$, where $f = f^{\rm Pre}$ for Pre-LN and $f= f^{\rm Peri} $ for Peri-LN, and $f^{\rm Pre}$ and $f^{\rm Peri}$ are defined in \eqref{eq:Pre-LN} and \eqref{eq:Peri-LN}, respectively.
Here, the expectation is taken over the input-output pairs $(\bfX_0, \bfy)$, and $\Delta t >0$ defines the time increment. For standard Transformers, $\Delta t=1$; for the Transformers with scaled residual steps proposed in~\Cref{sec:delta_t}, $\Delta t \in (0, 1)$.

In our theoretical analysis, we consider a non-parametric formulation of the training problem given by
\begin{align}\label{eq:discrete_time_OC_obj_nonpara_appendix}
    & \min_{\{f_i\} \subseteq \mathcal{U}_{\rm ad}} \; \mathbb{E}_{(\bfX_0, \bfy)} \; G(\bfX_D, \bfy) \\ 
    \begin{split} &\text{s.t.} \; \bfX_{i+\frac{1}{2}} = \bfX_i + \Delta t \cdot f_i(\bfX_{i}),
\end{split}
\end{align}
for $i=0,\frac{1}{2},1,...,D-\frac{1}{2}$. This formulation is non-parametric because the optimization is directly over $f_i$ instead of the model weights $\bftheta$. Here, $\mathcal{U}_{\text{ad}}$ denotes the admissible set of functions for $f$, which depends on the choice of layer normalization placement. For Pre-LN, it consists of functions that are representable by $f_{\rm attn}$ and $f_{\rm ffn}$ in (\ref{eq:self-attention}) and (\ref{eq:fully-connected}), respectively. For Peri-LN, Lemma~\ref{lemma:ellipsoid} implies that the admissible set consists of functions whose image lies in the ellipsoid defined therein. 

We first rewrite~\eqref{eq:discrete_time_OC_obj_nonpara_appendix}  as
\begin{align}\label{eq:discrete_time_OC_obj_nonpara_indicator_appendix}
\begin{split}
    \min_{ \{f_i\}} \quad& \mathbb{E}_{\bfy} \mathbb{E}_{(\bfX_0 | \bfy)} \left\{ G(\bfX_D, \bfy) + \Delta t \sum_{i} L_i(f_i(\bfX_i)) \right\} \\
    \text{s.t.} \quad& \bfX_{i+\frac{1}{2}}  = \bfX_i + \Delta t\cdot f_i(\bfX_{i}),
\end{split}
\end{align}
where $L_i$'s are indicator functions which equal to $0$ if $f_i \in \mathcal{U}_{\rm ad}$ and equal to $\infty$ otherwise. 

Next, we demonstrate the training formulation~\eqref{eq:discrete_time_OC_obj_nonpara_indicator_appendix} can be reduced into a (pointwise) optimal control problem, simplifying our analysis. We set up the notations as follows. For each target output $\bfy$, let $\mu_{\bfy}$ denote the conditional probability distribution of the inputs given $\bfy$. Given any input $\bfX_0$ and Transformer hidden layers $\{f_i\}_i$, we denote the single trajectory cost as
\begin{equation}
    J_\bfy(\bfX_0; \{f_i\}_i) := \Delta t \sum_{i} L_i(f_i(\bfX_i)) + G(\bfX_D, \bfy), \quad \bfX_{i+\frac{1}{2}} = \bfX_i + \Delta t \cdot f_i(\bfX_{i}),
\end{equation}
and the single trajectory optimal cost as 
\begin{equation}\label{eq:single_traj_opt_cost}
    \Upsilon_\bfy(\bfX_0):= \inf_{\{f_i\}_i} J_\bfy(\bfX_0; \{f_i\}_i).
\end{equation}

\begin{proposition}[Reduction to pointwise optimal control.]\label{prop:pointwise} Assume $\mu_\bfy$'s are integrable for each $\bfy$, and for each $\bfy$ and $\bfX_0$, the single trajectory optimal cost~\eqref{eq:single_traj_opt_cost} is attained at some $\{f_i^*\}_i$, then the mean-field training problem reduces to the pointwise integral. In particular, for each $\bfy$,
\begin{equation}\label{eq:pointwise_reduce}
    \inf_{\{f_i\}_i} \int J_\bfy(\bfX_0; \{f_i\}_i) \, d\mu_\bfy(\bfX_0) = \int \Upsilon_\bfy(\bfX_0) \, d\mu_\bfy(\bfX_0).
\end{equation}
\end{proposition}

\begin{remark}
    While Proposition~\ref{prop:pointwise} assumes the existence of the pointwise minimizers $\{f_i^*\}_i$, we later show in Proposition~\ref{prop:Bellman} that this existence is always guaranteed.
\end{remark}

\begin{remark}
    Proposition~\ref{prop:pointwise} shows that the pointwise optimum $\{ f_i^*\}_i$ also solves the mean field control optimal control problem. Thus, itallows us to simplify the analysis by shifting the focus from the integrated mean-field objective to the individual pointwise optimal control objective in \eqref{eq:single_traj_opt_cost}.
\end{remark}

\begin{proof}
    We show the two inequalities in~\eqref{eq:pointwise_reduce}. 

    \textbf{Lower bound.} Fix any $\{f_i\}_i$. For every $\bfX_0$, $J_\bfy(\bfX_0; \{f_i\}_i) \geq \Upsilon_\bfy(\bfX_0)$ by the definition of infimum. Integrating against $\mu_\bfy$ and taking the infimum over $\{f_i\}_i$ on the left,
    \[
    \inf_{\{f_i\}_i} \int J_\bfy(\bfX_0; \{f_i\}_i) \, d\mu_\bfy(\bfX_0) \geq \int \Upsilon_\bfy(\bfX_0) \, d\mu_\bfy(\bfX_0).
    \]

    \textbf{Upper bound.} The infimum on the left of~\eqref{eq:pointwise_reduce} satisfies 
    \begin{align*}
        \inf_{\{f_i\}_i} \int J_\bfy(\bfX_0; \{f_i\}_i) \, d\mu_\bfy(\bfX_0) \leq \int J_\bfy(\bfX_0; \{f_i^*\}_i) \, d\mu_\bfy(\bfX_0),
    \end{align*}
    where $\{f_i^*\}_i$ is the pointwise optimum. Moreover, by the definition of the single trajectory optimal cost~\eqref{eq:single_traj_opt_cost}, we have
    \[
    \int J_\bfy(\bfX_0; \{f_i^*\}_i) \, d\mu_\bfy(\bfX_0) = \int \Upsilon_\bfy(\bfX_0) \, d\mu_\bfy(\bfX_0).
    \]
    Combining these two results, we obtain the upper bound
    \[
    \inf_{\{f_i\}_i} \int J_\bfy(\bfX_0; \{f_i\}_i) \, d\mu_\bfy(\bfX_0) \leq \int \Upsilon_\bfy(\bfX_0) \, d\mu_\bfy(\bfX_0).
    \]
\end{proof}

\subsection{Assumptions}
In our derivations, we use the following assumptions.
\begin{enumerate}[label=(A\arabic*)]
    \item \label{assump:one}The functions $L_i: \mathbb{R}^{d\times n} \to \mathbb{R} \cup \{\infty\}$ are convex, proper, and lower-semicontinuous (l.s.c.),
    \item \label{assump:two}$L_i$'s are coercive, that is,
    \[
    \frac{L_i(f)}{\| f\|_F} \to \infty \quad \text{as $\| f\|_F^2 \to \infty$ }
    \]
    \item \label{assump:three}$L_i$'s are convex and satisfies $L_i(\bfX) \geq 0$ for all $i$ and $\bfX \in \mathbb{R}^{d \times n}$, and 
    \item \label{assump:four}the terminal cost $G(\cdot, \bfy)$ is continuous and bounded below: there exists $0< M < \infty$ such that $G(\cdot, \bfy) \geq -M$ for all $\bfX \in \mathbb{R}^{d \times n}$.
\end{enumerate}
We remark that for Pre-LN, $L_i$'s are the constant zero function, and for Peri-LN, $L_i$'s are an indicator function over an ellipsoid (\Cref{lemma:ellipsoid}), which is a strictly convex set. In these two cases, the first assumption is satisfied. 

The second assumption is obviously satisfied by Peri-LN. However, it is not satisfied by Pre-LN. This suggests the ill-posedness of Pre-LN training and will be discussed in detail later in our derivation.

The third assumption obviously holds for Pre-LN and Peri-LN.

The fourth assumption regards the terminal cost and is satisfied regardless of the layer normalization placement. For more details, see~\cite[Appendix F.1]{kan2025optimal}.

\subsection{Bellman's Equation}
For any intermediate layer $i=0,\frac{1}{2},1,...,D-\frac{1}{2}$, hidden state $\bfX$ and target output $\bfy$, we define the value function $\Phi_{i, \bfy}(\bfX)$ as the optimal cost of completing the trajectory from layer $i$ onward. Specifically, it is given by
\begin{equation}\label{eq:value_function}
    \Phi_{i, \bfy}(\bfX) = \inf_{\{ f_j\}_{j \geq i} } \left\{ G(\bfX_D, \bfy) + \Delta t \sum_{j \geq i} L_j(f_j(\bfX_j)) \; \middle| \; \bfX_i=\bfX \right\},
\end{equation}
and $\Phi_{D, \bfy}(\bfX) = G(\bfX, \bfy)$.

\begin{proposition}[Bellman Equation]\label{prop:Bellman}
    Under assumptions~\ref{assump:one}-\ref{assump:four}, the value function~\eqref{eq:value_function} satisfies the Bellman equation
    \begin{equation}\label{eq:value_function_recursive}
        \Phi_{i, \bfy}(\bfX) = \inf_{f} \left\{ \Phi_{i+\frac{1}{2}, \bfy}(\bfX + \Delta t \cdot f) + \Delta t \cdot L_i(f(\bfX)) \right\},
    \end{equation}
    for $i=0,\frac{1}{2},1,...,D-\frac{1}{2}$. Moreover, the infimum is attained: for every $\bfX \in \mathbb{R}^{d \times n}$ and every $i$, there exists an optimal velocity achieving the infimum.
\end{proposition}

\begin{proof}
    Fix $i \in \{ 0,\frac{1}{2},1,...,D-\frac{1}{2} \}$, and recall that we have the given terminal condition $\Phi_{D, \bfy}(\bfX) = G(\bfX, \bfy)$. Denote the right-hand-side of~\eqref{eq:value_function_recursive} by
    \begin{equation}\label{eq:value_function_recursive_psi}
        \tilde{\Phi}_{i, \bfy}(\bfX) := \inf_f \Psi_{\bfX, \bfy}(f), \quad \Psi_{\bfX, \bfy}(f) :=  \Phi_{i+\frac{1}{2}, \bfy}(\bfX + \Delta t \cdot f) + \Delta t \cdot L_i(f(\bfX)) .
    \end{equation}
    Our goal is to show that $\tilde{\Phi}_{i, \bfy}(\bfX) = {\Phi}_{i, \bfy}(\bfX)$ and that the infimum is attained.

    \textbf{Step 1} (${\Phi}_{i, \bfy}(\bfX) \geq \tilde{\Phi}_{i, \bfy}(\bfX)$)\textbf{.} Let $\bfX=\bfX_i$ and
    $\{ f_j\}_{j \geq i}$ be any control sequence.  Denote $f:=f_i$ as the first control. The total cost from~\eqref{eq:value_function} can be split as follows
    \begin{equation}
        \Delta t \sum_{j \geq i} L_j(f_j(\bfX_j)) + G(\bfX_D, \bfy)  = \Delta t \cdot L_i(f(\bfX)) + \underbrace{\Delta t \sum_{j \geq i+\frac{1}{2}} L_j(f_j(\bfX_j)) + G(\bfX_D, \bfy)}_{\text{tail cost from $\bfX_{i+\frac{1}{2}}$}}. 
    \end{equation}
    By~\eqref{eq:value_function} and the definition of infimum, the tail cost from $\bfX_{i+\frac{1}{2}}$ is at least $\Phi_{i+\frac{1}{2}, \bfy}(\bfX + \Delta t \cdot f)$. Thus we have
    \[
    \Delta t \sum_{j \geq i} L_j(f_j(\bfX_j)) + G(\bfX_D, \bfy) \geq \Delta t \cdot L_i(f(\bfX)) + \Phi_{i+\frac{1}{2}, \bfy}(\bfX + \Delta t \cdot f) = \Psi_{\bfX, \bfy}(f) \geq \tilde{\Phi}_{i, \bfy}(\bfX) .
    \]
    Taking infimum over all control sequence yields ${\Phi}_{i, \bfy}(\bfX) \geq \tilde{\Phi}_{i, \bfy}(\bfX)$.

    \textbf{Step 2} (${\Phi}_{i, \bfy}(\bfX) \leq \tilde{\Phi}_{i, \bfy}(\bfX)$)\textbf{.} Fix any $f \in \mathbb{R}^{d \times n}$ and $\epsilon > 0$. Set $\bfZ=\bfX + \Delta t \cdot f$. By definition of infimum and~\eqref{eq:value_function}, there exists a tail control sequence $\{ f_{j}^\epsilon \}_{j \geq i+\frac{1}{2}}$ starting from state $\bfX_{i+\frac{1}{2}}=\bfZ$ with tail cost at most $\Phi_{i+\frac{1}{2}, \bfy}(\bfZ) + \epsilon$. Consider a full control sequence from $(i, \bfX)$ which uses $f$ at layer $i$ and the tail control sequence $\{ f_{j}^\epsilon \}_{j \geq i+\frac{1}{2}}$ for all subsequent layers. This yields a hidden state trajectory with $\bfX_i=\bfX$, $\bfX_{i+\frac{1}{2}}=\bfZ$ and total cost
    \begin{equation}
        \Delta t \cdot L_i(f(\bfX)) + [\text{tail cost from $\bfZ$}] \leq \Delta t \cdot L_i(f(\bfX)) + \Phi_{i+\frac{1}{2}, \bfy}(\bfZ) +\epsilon = \Psi_{\bfX, \bfy}(f) + \epsilon.
    \end{equation}
    We remark that the constructed control sequence starts at $\bfX_i=\bfX$ and layer $i$, and the value function~\eqref{eq:value_function} is defined as the infimum cost, also starting at $\bfX_i=\bfX$ and layer $i$, thus we have
    \begin{equation}
        \Phi_{i, \bfy}(\bfX) \leq \Psi_{\bfX, \bfy}(f) + \epsilon.
    \end{equation}
    Since $\epsilon$ is arbitrary, taking $\epsilon \to 0$ yields
    \begin{equation}
        \Phi_{i, \bfy}(\bfX) \leq \Psi_{\bfX, \bfy}(f)
    \end{equation}
    Morevoer, taking infimum over the constructed sequence yields
    \begin{equation}
        \Phi_{i, \bfy}(\bfX) \leq \inf_f\Psi_{\bfX, \bfy}(f) =  \tilde{\Phi}_{i, \bfy}(\bfX).
    \end{equation}

    \textbf{Step 3 (Attainment of the infimum).} We show that the infimum in the Bellman equation
    \[
    {\Phi}_{i, \bfy}(\bfX) = \inf_{f} \left\{ \Phi_{i+\frac{1}{2}, \bfy}(\bfX + \Delta t \cdot f) + \Delta t \cdot L_i(f(\bfX)) \right\}
    \]
    is attained for some $f$. The arguments rely on the coercivity of $L_i$ (Assumption~\ref{assump:two}).

    Our argument uses a backward inductive argument starting with $i=D-\frac{1}{2}$, by the terminal condition $\Phi_{D, \bfy}(\bfX) = G(\bfX, \bfy)$, we have
    \begin{equation}\label{eq:Bellman_terminal}
    {\Phi}_{D-\frac{1}{2}, \bfy}(\bfX) = \inf_{f} \left\{ G(\bfX + \Delta t \cdot f, \bfy) + \Delta t \cdot L_{D-\frac{1}{2}}(f(\bfX)) \right\}.
    \end{equation}
    In the objective function, the first term $G$ is continuous and lower bounded by $-M>-\infty$ (assumption~\ref{assump:four}), and the second term $\Delta t \cdot L_{D-\frac{1}{2}}$ is ls.c., coercive, and lowered bounded by 0 (assumptions \ref{assump:one}, \ref{assump:two}, and \ref{assump:three}). Thus, the objective $ f \mapsto G(\bfX + \Delta t \cdot f, \bfy) + \Delta t \cdot L_{D-\frac{1}{2}}(f(\bfX))$ is coercive, bounded below and l.s.c.. This implies that, for any $c>0$, the sublevel set 
    \[
    \left\{ f\middle| G(\bfX + \Delta t \cdot f, \bfy) + \Delta t \cdot L_{D-\frac{1}{2}}(f(\bfX)) \leq c \right\}
    \]
    is bounded (coercivity) and closed (by l.s.c.) and hence compact. Thus, the Weierstrass extreme value theorem~\tr{cite} implies that the infimum is attained at some $f$. 

    In addition, the infimum objective in~\eqref{eq:Bellman_terminal} is jointly l.s.c., and the infimum of $f$ is taken over a compact set (due to the coercivity of $L$), ${\Phi}_{D-\frac{1}{2}, \bfy}(\bfX)$ is l.s.c.. Moreover, ${\Phi}_{D-\frac{1}{2}, \bfy}(\bfX) = \inf_{f} \{ G(\bfX + \Delta t \cdot f, \bfy) + \Delta t \cdot L_{D-\frac{1}{2}}(f(\bfX)) \} \geq 0-M=-M$, so it is bounded below. Thus, we can repeat the same arguments to iteratively show that the infimum is attained at some $f$ for ${\Phi}_{i, \bfy}(\bfX) $ for $i=D-1, D-\frac{3}{2},...,0$. 

    We remark that, we see that in the proof of attainment of infimum, a key assumption is the coercivity of $L_i$'s. However, recall that under Pre-LN, $L_i's$ are the constant zero function since no restriction is imposed onto the velocity field (output of the layers). Thus, under Pre-LN, the coercivity assumption does not hold.

    In more details, in the Pre-LN training problem, we only have the terminal cost $G$ and no running cost $L_i$'s to regulate the  velocity $f_i$'s. Thus, there are infinitely many velocity $f_i$'s that can attain the infimum of the Bellman equation~\eqref{eq:value_function_recursive}, including highly irregular ones. For instance, since $L_i$'s are not coercive, the velocity can have unbounded magnitude. This is the fundamental reason of the ill-posedness of the Pre-LN training problem.
\end{proof}

\subsection{Hamiltonian and Hamilton-Jacobi-Bellman Equation}
We first define the Hamiltonian $H_i$ to be the Legendre transform of $L_i$
\begin{equation}\label{eq:generic_Hamiltonian_appendix}
    H_i(\bfP) = \sup_{\bfV \in \mathbb{R}^{d \times n}} \left\{ \langle \bfP, \bfV \rangle - L_i(\bfV) \right\}.
\end{equation}

\begin{theorem}[Hamiltonian Form of the Bellman Equation]
    For each $i=0,\frac{1}{2},...,D-\frac{1}{2}$ and $\bfX \in \mathbb{R}^{d \times n}$, let $f_i^*$ attains the infimum in~\eqref{eq:value_function_recursive}, and set $T_i(\bfX) = \bfX + \Delta t \cdot f_i^*(\bfX)$. Assume that $\Phi_{i+\frac{1}{2}, \bfy}$ is differentiable at $T_i(\bfX)$, then
    \begin{enumerate}
        \item \textbf{Optimality condition.}
        \begin{equation}\label{eq:velocity_optimal}
            - \nabla \Phi_{{i+\frac{1}{2}, \bfy}}(T_i(\bfX)) \in \partial_f L_i(f_i^*(\bfX)).
        \end{equation}
        \item \textbf{Frenchel-Young equality.} Denoting $\bfP := -\nabla \Phi_{i+\frac{1}{2}, \bfy}(T_i(\bfX))$, we have
        \begin{equation}\label{eq:Frenchel_Young}
            L_i(f_i^*(\bfX)) + H_i(\bfP) = \langle \bfP, f_i^*(\bfX) \rangle.
        \end{equation}
        \item \textbf{Discrete-time Hamilton-Jacobi-Bellman partial differential equation (HJB PDE)}
        \begin{align}\label{eq:HJB_generic_appendix}
\begin{split}
    \Phi_{i, \bfy}(\bfX) &= \Phi_{i+\frac{1}{2}, \bfy}(T_i(\bfX)) - \Delta t \cdot \left[ \langle \nabla \Phi_{i+\frac{1}{2}, \bfy}(T_i(\bfX)), f_i^* \rangle + H_i(-\nabla \Phi_{i+\frac{1}{2}, \bfy}(T_i(\bfX))) \right]. \\
    \Phi_{D,\bfy}(\bfX) &=G(\bfX, \bfy).
    \end{split}
\end{align}
    \end{enumerate}
\end{theorem}

    \begin{remark}[Non-smooth value functions]\label{rem:visc_nonsmooth}
In this theorem, we assumed that the value function $\Phi_{i+\frac{1}{2}, \bfy}$ is differentiable at $T_i(\bfX)$ to simplify the derivation. This requirement is not strictly necessary and is adopted here primarily for ease of exposition. \textbf{The discrete-time HJB PDE still holds in general as a viscosity-style weak equation}, with equality replaced by sub/superinequalities
against smooth test functions touching from
above or below, in the spirit of viscosity
solutions~\cite{barles1991convergence,falcone2014semilagrangian}. \textbf{In turn, our analysis using the discrete-time HJB PDE remains valid even when the value function is non-differentiable.}
\end{remark}

\begin{proof}
    \textbf{(i).} Since $f_i^*$ minimizes the Bellman objective~\eqref{eq:value_function_recursive_psi}, its first order optimality condition gives
    \begin{equation}
         0 \in \Delta t \nabla  \Phi_{i+\frac{1}{2}, \bfy}(\bfX + \Delta t \cdot f_i^*) + \Delta t \cdot \partial_f L_i(f_i^*(\bfX)).
    \end{equation}
    Dividing by $\Delta t>0$ and rearranging the terms give~\eqref{eq:velocity_optimal}.

    \textbf{(ii).} We show that the optimality condition~\eqref{eq:velocity_optimal} implies the Frenchel-Young equality~\eqref{eq:Frenchel_Young}.

    \textit{Step 1 (upper bound):} By~\eqref{eq:generic_Hamiltonian_appendix} and the definition of the supremum
    \begin{align*}
        H_i(\bfP) &\geq \langle \bfP, f_i^*(\bfX) \rangle - L_i(f_i^*(\bfX))\\
        L_i(f_i^*(\bfX)) + H_i(\bfP) &\geq \langle \bfP, f_i^*(\bfX) \rangle
    \end{align*}

    \textit{Step 2 (lower bound):} For any $\bfV \in \mathbb{R}^{d \times n}$, the convexity of $L_i$ (assumption~\ref{assump:three}) gives 
    \[
    L_i(\bfV) \geq L_i(f_i^*(\bfX)) + \langle \xi, \bfV - f_i^*(\bfX) \rangle,
    \]
    for all $\xi \in \partial_f L_i(f^*_i(\bfX))$. 
    Combining with the optimality condition~\eqref{eq:velocity_optimal} yields
    \begin{align*}
    L_i(\bfV) &\geq L_i(f_i^*(\bfX)) - \langle \nabla \Phi_{{i+\frac{1}{2}, \bfy}}(T_i(\bfX)), \bfV - f_i^*(\bfX) \rangle \\
    &= L_i(f_i^*(\bfX)) + \langle \bfP, \bfV - f_i^*(\bfX) \rangle, \quad \text{by definition of $\bfP$,} \\
    &= L_i(f_i^*(\bfX)) + \langle \bfP, \bfV \rangle - \langle \bfP, f_i^*(\bfX) \rangle \\
    \langle \bfP, f_i^*(\bfX) \rangle - L_i(f_i^*(\bfX)) &\geq \langle \bfP, \bfV \rangle - L_i(\bfV).
    \end{align*}
    Since this inequality holds for all $\bfV \in \mathbb{R}^{d \times n}$, taking supremum yields
    \begin{align*}
        \langle \bfP, f_i^*(\bfX) \rangle - L_i(f_i^*(\bfX)) &\geq \sup_{\bfV \in \mathbb{R}^{d \times n}} \left\{ \langle \bfP, \bfV \rangle - L_i(\bfV) \right\} = H_i(\bfP) \\
        L_i(f_i^*(\bfX)) + H_i(\bfP) &\leq \langle \bfP, f_i^*(\bfX) \rangle.
    \end{align*}
    Combining two bounds gives the Frenchel-Young equality~\eqref{eq:Frenchel_Young}.

    \textbf{(iii).} On one hand, the terminal condition of the HJB PDE is due to the terminal condition of the value function~\eqref{eq:value_function}. On the other hand, since $f_i^*$ attains the infimum in the Bellman equation~\eqref{eq:value_function_recursive}, 
    \begin{equation}\label{eq:Bellman_recursive_optimalf}
    \Phi_{i, \bfy}(\bfX) = \Phi_{i+\frac{1}{2}, \bfy}(T_i(\bfX)) + \Delta t \cdot L_i(f_i^*(\bfX)) .
    \end{equation}
    Rearranging the Frenchel-Young equality~\eqref{eq:Frenchel_Young}, we have
    \[
        L_i(f_i^*(\bfX))  = \langle \bfP, f_i^*(\bfX) \rangle -H_i(\bfP) .
    \]
    Substituting this into~\eqref{eq:Bellman_recursive_optimalf}, we obtain
    \begin{align*}
        \Phi_{i, \bfy}(\bfX) &= \Phi_{i+\frac{1}{2}, \bfy}(T_i(\bfX)) + \Delta t \cdot \left( \langle \bfP, f_i^*(\bfX) \rangle -H_i(\bfP) \right) \\
        \Phi_{i, \bfy}(\bfX) &= \Phi_{i+\frac{1}{2}, \bfy}(T_i(\bfX)) - \Delta t \cdot \left[ \langle \nabla \Phi_{i+\frac{1}{2}, \bfy}(T_i(\bfX)), f_i^*(\bfX) \rangle + H_i(-\nabla \Phi_{i+\frac{1}{2}, \bfy}(T_i(\bfX))) \right].
    \end{align*}
\end{proof}

\section{Optimality Conditions of Pre-LN Transformer Training are not Well-defined}
We restate and prove~\Cref{thm:Pre-LN_illposed}.

\settheoremnumber{thm:Pre-LN_illposed}
\begin{theorem}
    The optimal solution $f^{\rm Pre}$ to the training problem (\ref{eq:discrete_time_OC_obj}-\ref{eq:discrete_time_OC_dynamics}) is unbounded in magnitude.
\end{theorem}
\restoretheoremnumber

\begin{proof}
This has been proven at the end of the proof of Proposition~\ref{prop:Bellman} using a dynamic programming argument. Here, we provide a more straight-forward argument by directly inspecting the HJB PDE.

Consider the non-parametric training problem
\begin{align}\label{eq:PRELN_nonpara_appendix}
    & \min_{\{f_i\} \subseteq \mathcal{U}_{\rm Pre}} \; \mathbb{E}_{(\bfX_0, \bfy)} \; G(\bfX_D, \bfy) \\ 
    \begin{split} &\text{s.t.} \; \bfX_{i+\frac{1}{2}} = \bfX_i + f_i(\bfX_{i}),
\end{split}
\end{align}
for $i=0,\frac{1}{2},1,...,D-\frac{1}{2}$. 
Here, $\mathcal{U}_{\rm Pre}$ is the admissible set of functions corresponding to the self-attention and feedforward sublayers $f_{\rm attn}$ and $f_{\rm ffn}$, defined in (\ref{eq:self-attention})–(\ref{eq:fully-connected}) for Pre-LN Transformers. 

By the derivations in~\Cref{sec:MFC}, the corresponding HJB PDE reads 
\begin{align}\label{eq:HJB_Pre_appendix}
\begin{split}
    \Phi_{i,\bfy}(\bfX) &= \Phi_{i+\frac{1}{2},\bfy}(T_i(\bfX)) - \Delta t\left[ \langle \nabla \Phi_{{i+\frac{1}{2}, \bfy}}(T_i(\bfX)), f_i(\bfX) \rangle + H_i(- \nabla \Phi_{{i+\frac{1}{2}, \bfy}}(T_i(\bfX))) \right]. \\
    \Phi_{D,\bfy}(\bfX) &=G(\bfX, \bfy).
    \end{split}
\end{align}
where by~\eqref{eq:generic_Hamiltonian_appendix}, the Hamiltonian $H$ is given by
\begin{align*}\label{eq:Hamiltonian_Pre}
    H_i(\bfP) &= \sup_{f \in \mathbb{R}^{d \times n}} \left\{ \langle \bfP, f \rangle - L_i^{\rm Pre}(f) \right\} \\
    &= \sup_{f \in \mathcal{U}_{\rm Pre}} \left\{ \langle \bfP, f \rangle \right\}.
\end{align*}

We remark that for Pre-LN Transformers, the layer normalization operation is applied to the inputs only. Thus, the output magnitude is unconstrained, and $\mathcal{U}_{\rm Pre}$ contains functions of unbounded magnitude. Consequently, the supremum in the Hamiltonian~(\ref{eq:Hamiltonian_Pre}) can be made arbitrarily large in magnitude by scaling $f^{\rm Pre}$ in the direction of $\bfP$. In other words, the magnitude of the Hamiltonian's maximizer (the optimal solution to the training problem) is unbounded.

Under such degeneracy, the corresponding HJB PDE (\ref{eq:HJB_Pre_appendix}) is not well-defined. This means that the training problem is degenerate: there exist infinitely many choices of $f^{\rm Pre}$ that minimize the training objective, including functions with unbounded magnitude.
\end{proof}

\section{Exponential Growth Pre-LN Transformer Hidden States Under Weight Decay}
We restate and prove~\Cref{thm:PreLN_expo_growth}
\settheoremnumber{thm:PreLN_expo_growth}
\begin{theorem}
    For a Pre-LN Transformer trained with weight decay, given an input $\bfX_0$, the mean absolute value of the terminal hidden states ${\rm MA}(\bfX_D)$ satisfies
\begin{align}
        {\rm MA}(\bfX_D) \leq \frac{1}{\sqrt{nd}} \, \bigl(1 + C(\lambda)\bigr)^D \, \|\bfX_0\|_F = \mathcal{O}(e^D),
    \end{align}
    where $\| \cdot \|_F$ denotes the Frobenius norm, $C$ is a constant whose magnitude depends on the weight decay hyperparameter $\lambda$.
\end{theorem}
\restoretheoremnumber

\begin{proof}
    To avoid unnecessary technical clutter, we assume the following simplifications to the Pre-LN Transformer model. These assumptions are made without loss of generality, as our arguments extend directly to the full model.
    \begin{enumerate}
        \item The Transformer blocks only contain the self-attention sublayer~(\ref{eq:self-attention})
        \item The self-attention sublayers has a single head, i.e., $H=1$
        \item RMSNorm~(\ref{eq:RMSNorm}) is used for layer normalization
    \end{enumerate}
    Thus, the hidden state at the $i$th Transformer block is given by
    \begin{align}
        \bfX_{i+1} &= \bfX_i + f_{\rm attn}( {\rm RMSNorm}(\bfX_i; \bfgamma_i) )\\
        &= \bfX_i + f_{\rm attn}( \bfX_i \bfGamma \bfD_i^{-1} ) \\
        &= \bfX_i + \underbrace{\bfW_i^1 \bfV_i^1}_{=:\bfW_i} \bfX_i \bfGamma \bfD_i^{-1} \underbrace{{\rm softmax} \left( \frac{(\bfK_i^1 \bfX_i \bfGamma \bfD_i^{-1})^\top \bfQ_i^1 \bfX_i \bfGamma \bfD_i^{-1}}{\sqrt{k}} \right)}_{=:\bfA_i} \\
        &= \bfX_i + \bfW_i {\bfX}_i \bfGamma \bfD_i^{-1} \bfA_i \label{eq:PreLN_Tran_sim},
    \end{align}
    where the second step follows from applying RMSNorm~(\ref{eq:RMSNorm}) column-wise to $\bfX_i$, with $\bfGamma= {\rm diag} \left( \bfgamma_i \right)$ and $\bfD_i = {\rm diag} \left( \|\bfx_{i,1}\|_2, \|\bfx_{i,2}\|_2, ..., \|\bfx_{i,n}\|_2 \right)$.

    By the Sylvester matrix equation, the vectorized form of (\ref{eq:PreLN_Tran_sim}) can be written as
    \begin{align}
        {\rm vec}(\bfX_{i+1}) &= {\rm vec}(\bfX_{i})  + \left( (\bfA_i^\top \bfD_i^{-1} \bfGamma) \otimes \bfW_i \right){\rm vec}(\bfX_{i}) \\
        &= \left(\bfI + \left(  (\bfA_i^\top \bfD_i^{-1} \bfGamma) \otimes \bfW_i \right) \right) {\rm vec}(\bfX_{i}),
    \end{align}
    where $\otimes$ denotes the Kronecker product. Expanding the recursion yields the formulation for the last hidden states (of the $D$th block)
    \begin{equation}
        {\rm vec}(\bfX_{D}) = \prod_{i=0}^{D-1} \left(\bfI + \left(  (\bfA_i^\top \bfD_i^{-1} \bfGamma) \otimes \bfW_i \right) \right) {\rm vec}(\bfX_0).
    \end{equation}
    We now derive the corresponding relation in terms of Frobenius norm
    \begin{align}
        \| \bfX_D\|_F &= \| {\rm vec}(\bfX_D) \|_2 \\
        &= \left\| \prod_{i=0}^{D-1} \left(\bfI + \left(  (\bfA_i^\top \bfD_i^{-1} \bfGamma) \otimes \bfW_i \right) \right) {\rm vec}(\bfX_0) \right\|_2 \\
        &\leq \left\| \prod_{i=0}^{D-1} \left(\bfI + \left(  (\bfA_i^\top \bfD_i^{-1} \bfGamma) \otimes \bfW_i \right) \right) \right\|_2 \left\|{\rm vec}(\bfX_0)\right\|_2 \\
        &= \left\| \prod_{i=0}^{D-1} \left(\bfI + \left(  (\bfA_i^\top \bfD_i^{-1} \bfGamma) \otimes \bfW_i \right) \right) \right\|_2 \left\|\bfX_0\right\|_F \\
        &\leq \prod_{i=0}^{D-1} \left\|  \left(\bfI + \left(  (\bfA_i^\top \bfD_i^{-1} \bfGamma) \otimes \bfW_i \right) \right) \right\|_2 \left\|\bfX_0\right\|_F \\
        & \text{by sub-multiplicativity of operator norm,}\\
        & \leq \prod_{i=0}^{D-1}\left( \left\|  \bfI \right\|_2 +  \left\| (\bfA_i^\top \bfD_i^{-1} \bfGamma) \otimes \bfW_i \right\|_2 \right)   \left\|\bfX_0\right\|_F, \\
        & \text{by triangle inequality,}\\
        &= \prod_{i=0}^{D-1}\left( 1 +  \left\| \bfA_i^\top \bfD_i^{-1} \bfGamma \right\|_2 \left\| \bfW_i \right\|_2 \right)   \left\|\bfX_0\right\|_F, \\
        & \text{by a property of Kronecker product,} \\
        & \leq \prod_{i=0}^{D-1}\left( 1 +  \underbrace{\left\| \bfA_i\right\|_2}_{\leq \sqrt{n}} \underbrace{\left\| \bfD_i^{-1} \right\|_2}_{\leq \max_j \left(\frac{1}{\| \bfx_{i,j} \|_2} \right)} \underbrace{\left\| \bfGamma \right\|_2}_{= \left\| \bfgamma_i \right\|_\infty} \left\| \bfW_i \right\|_2 \right)   \left\|\bfX_0\right\|_F \\
        &\leq  \prod_{i=0}^{D-1} \left( 1 +  { \sqrt{n}} \left\| \bfgamma_i \right\|_\infty { \max_j \left(\frac{1}{\| \bfx_{i,j} \|_2} \right)}  \left\| \bfW_i \right\|_2 \right)   \left\|\bfX_0\right\|_F.
    \end{align}
    Using the fact that ${\rm MA}(\bfX_D) \leq \frac{1}{\sqrt{nd}} \| \bfX_D\|_F$, we have
    \begin{equation}
        {\rm MA}(\bfX_D) \leq \prod_{i=0}^{D-1} \frac{1}{\sqrt{nd}}\left( 1 +  { \sqrt{n}} \left\| \bfgamma_i \right\|_\infty { \max_j \left(\frac{1}{\| \bfx_{i,j} \|_2} \right)}  \left\| \bfW_i \right\|_2 \right)   \left\|\bfX_0\right\|_F
    \end{equation}
    The product over $i$ can result in exponential growth in $\| \bfX_i \|_F$ especially when $\| \bfW_i \|_2$ is large in magnitude, which happens when the weight decay on $\bfW_i$ is not sufficient. Specifically, in such case we have
    \begin{equation}
        {\rm MA}(\bfX_D) \leq \frac{1}{\sqrt{nd}} \, \bigl(1 + C(\lambda)\bigr)^D \, \|\bfX_0\|_F = \mathcal{O}(e^D).
    \end{equation}
    
\end{proof}

\section{Optimality Conditions of Peri-LN Transformer Training are Well-defined}
For each fixed set of layer normalization parameters $(\bfgamma, \bfbeta)$, we consider the Peri-LN Transformer training formulation 
\begin{align}\label{eq:PERILN_nonpara_appendix}
    & \min_{\{f_i\} \subseteq \mathcal{U}_{\rm Peri}} \; \mathbb{E}_{(\bfX_0, \bfy)} \; G(\bfX_D, \bfy) \\ 
    \begin{split} &\text{s.t.} \; \bfX_{i+\frac{1}{2}} = \bfX_i + f_i(\bfX_{i}),
\end{split}
\end{align}
for $i=0,\frac{1}{2},1,...,D-\frac{1}{2}$. 
Moreover, $\mathcal{U}_{\rm Peri}$ is the admissible set of functions corresponding to the self-attention and feedforward modules $f_{\rm attn}$ and $f_{\rm ffn}$, defined in (\ref{eq:self-attention})–(\ref{eq:fully-connected}) for Peri-LN Transformers. By Lemma~\ref{lemma:ellipsoid}, $\mathcal{U}_{\rm Peri}$ consists of functions $f_i: \mathbb{R}^{d \times n}  \to \mathbb{R}^{d \times n}$, where each of the $n$ columns of the output lies on an ellipsoid. More specifically, we have
\begin{equation}\label{eq:U_Peri}
    \mathcal{U}_{\rm Peri} = \left\{ f_i \; \middle| \; 
    \begin{aligned}
& ([f_i(\bfX)]_j - \bfbeta_{i, \rm out})^\top \bfGamma_{i, \rm out}^{-2} ([f_i(\bfX)]_j  - \bfbeta_{i, \rm out}) = d,\\
& \text{for } j=1,2,...,n, \; \text{where } [f(\bfX)]_j \text{ denotes the $j$th column of $f(\bfX)$},\\
& \bfGamma_{\rm out}^{-2} = \mathrm{diag}(\gamma_{{\rm out},1}^{-2}, \dots, \gamma_{{\rm out},d}^{-2}) 
  \in \mathbb{R}^{d \times d}, \;\text{with } \gamma_{{\rm out},k} \text{ denoting the $k$th entry of } \bfgamma_{\rm out}.
\end{aligned}
 \right\}.
\end{equation}

We consider the optimality conditions of the training problem, which contains an HJB PDE given by 
\begin{align}
\begin{split}
    \Phi_{i,\bfy}(\bfX) 
    &= \Phi_{i+\frac{1}{2},\bfy}(T_i(\bfX)) - \Delta t\left[ \langle \nabla \Phi_{{i+\frac{1}{2}, \bfy}}(T_i(\bfX)), f_i(\bfX) \rangle + H_i(- \nabla \Phi_{{i+\frac{1}{2}, \bfy}}(T_i(\bfX))) \right]. \\
    \Phi_{D,\bfy}(\bfX) &=G(\bfX, \bfy).
    \end{split}
\end{align}
where the Hamiltonian $H$ is given by
\begin{align}
    H_i(\bfP) &= \sup_{f \in \mathbb{R}^{d \times n}} \;\left\{ \langle \bfP, f \rangle - L_i^{\rm Peri}(f) \right\} \\
    &= \sup_{f \in \mathcal{U}_{\rm Peri}} \;\left\{ \langle \bfP, f \rangle \right\}  \\
    &= \sup_{f^{\rm Peri}\in \mathbb{R}^{d \times n}} \; \left\{ \langle \bfP, f^{\rm Peri} \rangle \middle| (f_j - \bfbeta_{i, \rm out})^\top \bfGamma_{i, \rm out}^{-2} (f_j - \bfbeta_{i, \rm out}) = d, \quad \text{for} \quad j=1,2,...,n \right\}, \label{eq:Hamiltonian_PeriLN}
\end{align}
by~\eqref{eq:U_Peri}, where $f_j$ denotes the $j$th column of $f$. 

Next, we show that the Hamiltonian~(\ref{eq:Hamiltonian_PeriLN}) admits a unique maximizer by explicitly deriving it. For brevity of notation, we drop the layer index $i$.
The maximization problem in~(\ref{eq:Hamiltonian_PeriLN}) is separable with respect to the columns of $f^{\rm Peri}$. 
Hence, the KKT conditions of the Hamiltonian read
\begin{align}
     -\bfp_j + 2\eta_j \bfGamma_{\rm out}^{-2} (f_j-\bfbeta_{\rm out})&= {\bf 0}, \label{eq:LN_noRun_KKT1} \\
    (f_j - \bfbeta_{\rm out})^\top \bfGamma_{\rm out}^{-2} (f_j - \bfbeta_{\rm out}) & = d,\label{eq:LN_noRun_KKT2}
\end{align}
for all $j\in[1,n]$, and where $\eta_j$ are the Lagrange multipliers for the constraints, and $\bfp_j$ is the $j$-th column of $\bfP$. We remark that $\eta_j > 0$ for all $j\in[1,n]$. This follows from the sensitivity interpretation of Langrange multipliers: increasing $d$ raises the optimal objective value of (\ref{eq:Hamiltonian_PeriLN}); see~\cite[Section~12.8]{nocedal2006numerical}. Since the objective is unbounded, increasing $d$ necessarily leads to a larger optimal value.

Rearrange~(\ref{eq:LN_noRun_KKT1}) to solve for $f_j$, we get
\begin{equation}\label{eq:LN_noRun_fi}
   f_j = \frac{1}{2\eta_j} \bfGamma^{2}_{\rm out} \bfp_j + \bfbeta_{\rm out},
\end{equation}
for all $j\in[1,n]$. Plugging this into~(\ref{eq:LN_noRun_KKT2}) and solve for $\eta_j$, we obtain
\begin{equation}
    \eta_j = \frac{\sqrt{\bfp_j^\top \bfGamma^{2}_{\rm out} \bfp_j}}{2 \sqrt{d}} \quad \text{or } - \frac{\sqrt{\bfp_j^\top \bfGamma^{2}_{\rm out} \bfp_j}}{2 \sqrt{d}} \quad \text{(rejected, since $\eta_j>0$)},
\end{equation}
for all $j\in[1,n]$. Plugging this into~(\ref{eq:LN_noRun_fi}), we obtain the \emph{unique} optimal solution to the Hamiltonian
\begin{equation}\label{eq:fjstar_Peri}
    f_j^* = \sqrt{\frac{d}{\bfp_j^\top \bfGamma^{2}_{\rm out} \bfp_j}} \bfGamma^{2}_{\rm out}\bfp_j + \bfbeta_{\rm out},
\end{equation}
for all $j\in[1,n]$. Collecting the column vectors $f_j^*$ yields the matrix ${f^{\rm Peri}}^*=[f^*_1, f^*_2, ..., f^*_n]$ given by
\begin{equation}\label{eq:fstar_Peri}
    f^*_{\rm Peri} = \sqrt{d} \bfGamma^{2}_{\rm out} \bfP {\rm diag} \left( \frac{1}{\sqrt{\bfP^\top \bfGamma^{2}_{\rm out} \bfP}} \right) + \bfbeta_{\rm out} \mathbf{1}_n^\top,
\end{equation}
where $\mathbf{1}_n \in \mathbb{R}^n$ is a vector of all ones, 
and the reciprocal is taken element-wise. We remark again that this optimal solution is \emph{unique}.

Finally, the optimal solution to the training problem is given by
\begin{align}
    {f^{\rm Peri}_i}^*(\bfX) &= \arg \sup_{f \in \mathcal{U}_{\rm Peri}} \left\{\langle - \nabla \Phi_{{i+\frac{1}{2}, \bfy}}(T_i(\bfX)), f \rangle \right\}\\
    &= - \sqrt{d} \bfGamma_{i, {\rm out}}^{2} \nabla \Phi_{{i+\frac{1}{2}, \bfy}}(T_i(\bfX)) {\rm diag} \left( \frac{1}{\sqrt{\nabla \Phi_{{i+\frac{1}{2}, \bfy}}(T_i(\bfX))^\top \bfGamma_{i, {\rm out}}^{2} \nabla \Phi_{{i+\frac{1}{2}, \bfy}}(T_i(\bfX))}} \right) + \bfbeta_{i, \rm out} \mathbf{1}_n^\top, \label{eq:optimal_velocity_Peri}
\end{align}
by (\ref{eq:fstar_Peri}).

\section{Forward Stability of Peri-LN Transformers}

\paragraph{Entry-wise Moments} We restate and prove~\Cref{thm:linear_growth_discrete}

\settheoremnumber{thm:linear_growth_discrete}
\begin{theorem}[Controlled Growth of Entry-wise Moments]
    Given an input $\bfX_0$, the mean absolute value $\rm MA$ and variance $\rm Var$ of the terminal hidden states $\bfX_D$ of a Peri-LN Transformer satisfy, respectively,
    \begin{align}
        {\rm MA}(\bfX_D) &\leq \frac{1}{\sqrt{nd}} \| \bfX_0 \|_F + 2 D (\gamma_{\rm max} + \beta_{\rm max}) = \mathcal{O}(D),\\
        \begin{split}
         {\rm Var}(\bfX_D) &\leq \frac{( \left\|  \bfX_{0} \right\|_F + 2D \sqrt{nd} (\gamma_{\rm max} + \beta_{\rm max}) )^2}{nd-1} = \mathcal{O}(D^2), 
    \end{split}
    \end{align}
where $\gamma_{\rm max} := \max\limits_{1 \leq i \leq D}  \{ \| \bfgamma^{\rm out}_{{\rm attn}, i} \|_\infty, \|\bfgamma^{\rm out}_{{\rm ffn}, i}\|_\infty \}$ takes the maximum over all layers and both attention and feedforward sublayers, with ``out'' denoting output layer normalization; similarly for $\beta_{\rm max}$.
\end{theorem}
\restoretheoremnumber

\begin{proof}
    To establish our proof, we note the following notations and facts
    \begin{itemize}
    \item Given an input $\bfX_0$, recall that the Peri-LN Transformer blocks read
        \begin{align}
        \bfU_{i} &= \bfX_i + f_{{\rm attn}, i}^{\rm Peri}(\bfX_i) = \bfX_i + {\rm LN}^{\rm out}_{{\rm attn}, i}(f_{{\rm attn}, i}({\rm LN}^{\rm in}_{{\rm attn}, i}(\bfX_i))) \\
     \bfX_{i+1} &= \bfU_{i} + f_{{\rm ffn}, i}^{\rm Peri}(\bfU_{i}) = \bfU_{i} + {\rm LN}_{{\rm ffn}, i}^{\rm out}(f_{{\rm ffn}, i}({\rm LN}_{{\rm ffn}, i}^{\rm in}(\bfU_{i}))), 
     \end{align}
     for $i=0,1,...,D-1$, and we use $\bfgamma^{\rm out}_{\rm attn}$ and $\bfbeta^{\rm out}_{\rm attn}$ to denote the parameters for ${\rm LN}^{\rm out}_{\rm attn}$ and use the same convention for the other LN. Using the relations above, we can obtain a formulation for $\bfX_0$ as follows
     \begin{align}
         \bfX_{D} &= \bfX_{D-1} + f_{\rm attn, D-1}^{\rm Peri}(\bfX_{D-1}) + f_{\rm ffn, D-1}^{\rm Peri}(\bfU_{D-1}) \\
         &= \bfX_{D-2} + \sum_{i=D-2}^{D-1} \left( f_{\rm attn, i}^{\rm Peri}(\bfX_{i}) + f_{\rm ffn, i}^{\rm Peri}(\bfU_{i})\right) \\
         &= \bfX_{0} + \sum_{i=0}^{D-1} \left( f_{\rm attn, i}^{\rm Peri}(\bfX_{i}) + f_{\rm ffn, i}^{\rm Peri}(\bfU_{i})\right) .\label{eq:expanded recursion}
     \end{align}
     \item Given a Peri-LN module $f^{\rm Peri} = {\rm LN}^{\rm out}(f({\rm LN}^{\rm in}(\cdot);\bftheta^{\rm ffn}_{i}))$. We denote the $j$th columns of $f$ and $f^{\rm Peri}$ as $f_j$ and $f^{\rm Peri}_j$, respectively. By the definition of layer normalization, we have 
     \begin{equation}\label{eq:f_j_LN_formula}
            f_j^{\rm Peri}(\bfX) = \bfgamma^{\rm out} \odot \hat{f}_j(\bfX) + \bfbeta^{\rm out},
        \end{equation}
        for $j=1,...,n$, where $\bfgamma^{\rm out}, \bfbeta^{\rm out} \in \mathbb{R}^d$ are parameters of ${\rm LN}^{\rm out}$, $\odot$ is the Hadamard element-wise product, $\hat{f}_j$ is given by
\begin{equation}\label{eq:standardized_f_appendix}
    \hat{f}_j = \frac{f_j - \mu}{\sqrt{\sigma^2 + \epsilon}}, \quad \text{with} \quad \mu = \frac{1}{d} \sum_{l=1}^d f_{lj}, \quad \sigma = \sqrt{\frac{1}{d} \sum_{l=1}^d (f_{lj} - \mu)^2},
\end{equation}
and $f_{lj}$ is the $l$th entry of $f_j$.
\item We have the inequality
        \begin{equation}\label{eq:int_term_bound}
        \left\|  f^{\rm Peri}(\bfX) \right\|_F \leq \sqrt{nd} (\gamma_{\rm max} + \beta_{\rm max}),
        \end{equation}
        for any $\bfX \in \mathbb{R}^{d \times n}$. This is because
        \begin{align*}
        \left\|  f^{\rm Peri}(\bfX) \right\|_F & =  \sqrt{\sum_{j=1}^n \left\| f^{\rm Peri}_j(\bfX) \right\|_2^2}\\
            &= \sqrt{\sum_{j=1}^n \left\|  \bfgamma^{\rm out} \odot \hat{f}_j(\bfX) + \bfbeta^{\rm out}\right\|_2^2} \\
            &\leq \sqrt{\sum_{j=1}^n \left( \left\|  \bfgamma^{\rm out} \odot \hat{f}_j(\bfX)\right\|_2 + \left\| \bfbeta^{\rm out}\right\|_2 \right)^2} \\
            &\leq \sqrt{\sum_{j=1}^n \left( \gamma_{\rm max}\left\|  \hat{f}_j(\bfX)\right\|_2 + \sqrt{d} \beta_{\rm max} \right)^2}, \quad \text{as } \|\bfbeta^{\rm out}\|_2 \leq \| \beta_{\rm max} {\bf 1}_d\|_\infty=\sqrt{d}\beta_{\rm max},  \\
            &= \sqrt{ \sum_{j=1}^n \left( \gamma_{\rm max} \sqrt{\frac{\sum_{l=1}^d (f_{lj}-\mu)^2}{\sigma^2+\epsilon}} + \sqrt{d} \beta_{\rm max} \right)^2}, \quad \text{by (\ref{eq:standardized_f_appendix})},\\
            &\leq \sqrt{ \sum_{j=1}^n \left( \gamma_{\rm max} \sqrt{\frac{\sum_{l=1}^d (f_{lj}-\mu)^2}{\sigma^2}} + \sqrt{d} \beta_{\rm max} \right)^2} \\
            &= \sqrt{ \sum_{j=1}^n \left( \sqrt{d} \gamma_{\rm max} + \sqrt{d} \beta_{\rm max} \right)^2} \quad \text{by (\ref{eq:standardized_f_appendix}) again,} \\
            &= \sqrt{ n \left( \sqrt{d} \gamma_{\rm max} + \sqrt{d} \beta_{\rm max} \right)^2} \\
            &= \sqrt{nd} (\gamma_{\rm max} + \beta_{\rm max}).
        \end{align*}
    \end{itemize} 

    Next, we derive the result. We have
    \begin{align}
        \| \bfX_D \|_F 
        &= \left\|  \bfX_{0} + \sum_{i=0}^{D-1} \left( f_{\rm attn, i}^{\rm Peri}(\bfX_{i}) + f_{\rm ffn, i}^{\rm Peri}(\bfU_{i})\right) \right\|_F, \quad \text{by (\ref{eq:expanded recursion})}, \nonumber\\
        & \leq \left\|  \bfX_{0} \right\|_F +  \sum_{i=0}^{D-1} \left\|  f_{\rm attn, i}^{\rm Peri}(\bfX_{i}) \right\|_F + \sum_{i=0}^{D-1} \left\|  f_{\rm ffn, i}^{\rm Peri}(\bfU_{i}) \right\|_F \nonumber \\
        &\leq \left\|  \bfX_{0} \right\|_F + 2D \sqrt{nd} (\gamma_{\rm max} + \beta_{\rm max}), \quad \text{by (\ref{eq:int_term_bound})}. \label{eq:X_D_bound_Fnorm}
    \end{align}

    By the Cauchy-Schwarz inequality, we have $\frac{1}{\sqrt{nd}} \| \bfX \|_1 \leq \| \bfX\|_F$ for any $\bfX \in \mathbb{R}^{d \times n}$, thus we have
    \begin{align*}
        \frac{1}{\sqrt{nd}} \| \bfX_D \|_1 \leq \| \bfX_D \|_F \leq \left\|  \bfX_{0} \right\|_F + 2D \sqrt{nd} (\gamma_{\rm max} + \beta_{\rm max}).
    \end{align*}
    This implies
    \begin{equation}
        {\rm MA}(\bfX_D) = \frac{1}{nd} \| \bfX_D \|_1 \leq \frac{1}{\sqrt{nd}} \| \bfX_0 \|_F + 2D (\gamma_{\rm max} + \beta_{\rm max}).
    \end{equation}
    Next, we prove the inequality for variance. Denote $[\bfX_D]_{ij} \in \mathbb{R}$ as the $ij$th entry of $\bfX_D \in \mathbb{R}^{d \times n}$, for $i=1,...,d$ and $j = 1,...,n$, and $\bar{\bfX}_D \in \mathbb{R}$ as the mean of all entries of $\bfX_D$. We have
\begin{align*}
{\rm Var}(\bfX_D) & = \frac{1}{nd-1} \sum_{i=1}^{d} \sum_{j=1}^{n} \left(  [\bfX_D]_{ij} - \bar{\bfX}_D \right)^2 \\
& = \frac{1}{nd-1} \left( \sum_{i=1}^{d} \sum_{j=1}^{n} [\bfX_D]_{ij}^2 - nd \bar{\bfX}_D^2 \right) \\
& \leq \frac{1}{nd-1} \sum_{i=1}^{d} \sum_{j=1}^{n} [\bfX_D]_{ij}^2   \\
&= \frac{1}{nd-1} \| \bfX_D\|_F^2 \\
& \leq \frac{1}{nd-1} \left( \left\|  \bfX_{0} \right\|_F + 2D \sqrt{nd} (\gamma_{\rm max} + \beta_{\rm max}) \right)^2, \quad \text{by (\ref{eq:X_D_bound_Fnorm})}.
\end{align*}
\end{proof}

\paragraph{Data-wise Variance} Then, we restate and prove~\Cref{thm:datawise_variance}. 

\settheoremnumber{thm:datawise_variance}

\begin{theorem}[Quadratic Growth of Data-wise Variance]
    Let $\bfX_0$ be an input and $\bfX_D$ its corresponding terminal hidden states. For each entry $x$ of $\bfX_D$, its variance across the data distribution satisfies
    \begin{equation}
    \begin{split}
        {\rm Var}(x) &\leq \mathbb{E}_{\bfX_0} \left[ \left(  \|\bfX_0\|_F + 2D\sqrt{nd}(\gamma_{\rm max} + \beta_{\rm max}) \right)^2\right] = \mathcal{O}(D^2),
    \end{split}
    \end{equation}
    where $ \gamma_{\rm max}$ and $\beta_{\rm max}$ are defined as in~\Cref{thm:linear_growth_discrete}, and the expectation is taken over $\bfX_0$ drawn from the data distribution.
\end{theorem}
\restoretheoremnumber

\begin{proof}
By definition,
\begin{align*}
    {\rm Var}(x) &= \mathbb{E}_{x} \left[ ( x -\mathbb{E}_x[x])^2 \right] \\
    & \leq \mathbb{E}_x \left[ x^2 \right] \\
    & \leq \mathbb{E}_{\bfX_D}  \left[ \|\bfX_D \|_F^2 \right] \\
    &\leq \mathbb{E}_{\bfX_0} \left[ \left(  \|\bfX_0\|_F + 2D\sqrt{nd}(\gamma_{\rm max} + \beta_{\rm max}) \right)^2\right], \quad \text{by \eqref{eq:X_D_bound_Fnorm}}.
\end{align*}
\end{proof}

\paragraph{Uncertainty Quantification} We restate and prove~\Cref{thm:uq}.

\settheoremnumber{thm:uq}

\begin{theorem}
    Let $\mu_0$ and $\nu_0$ be any two input distributions, $\mu_D$ and $\nu_D$ denote their pushforwards to the terminal hidden states under a Peri-LN Transformer, and $W^p_p(\mu,\nu) = \inf_{\gamma \in \Gamma(\mu,\nu)} \int_{\mathbb{R}^{dn} \times \mathbb{R}^{dn}} \|\bfX-\bfX'\|_p^p d\gamma $ to be the $p$-Wasserstein distance. There exists $\hat{C}(p)$ such that for any $p \geq 1$,
    \begin{equation}
        W_p(\mu_D,\nu_D) \leq 2^{\frac{p-1}{p}} \left( \hat{C}(p)W_p(\mu_0,\nu_0) +4 D \sqrt{nd}   \gamma_{\rm max}\right).
    \end{equation}
    Here, $ \gamma_{\rm max}$ is defined as in~\Cref{thm:linear_growth_discrete}. In contrast, for Pre-LN Transformers, the difference can be unbounded.
\end{theorem}
\restoretheoremnumber

\begin{proof}
    To establish our proof, we note the following notations and facts
    \begin{itemize}
\item Given any $\bfX \in \mathbb{R}^{d \times n}$, we have
\begin{align}
    \left\| \hat{f}_j(\bfX) \right\|_2 &=  \sqrt{\frac{\sum_{j=1}^n (f_j-\mu)^2}{\sigma^2 + \epsilon}} \quad \text{by (\ref{eq:standardized_f_appendix})} \nonumber \\
    &\leq \sqrt{\frac{\sum_{j=1}^n (f_j-\mu)^2}{\sigma^2 }} \nonumber \\
    &= \sqrt{d} \label{eq:col_f_bound}.
\end{align}
\item Given any $\bfX^{(1)}, \bfX^{(2)} \in \mathbb{R}^{d \times n}$, we have the inequality
        \begin{equation}\label{eq:diff_term_bound}
            \left\|  f^{\rm Peri}(\bfX^{(1)}) - f^{\rm Peri}(\bfX^{(2)}) \right\|_F \leq 2 \sqrt{nd} \gamma_{\rm max},
        \end{equation}
        for $j=1,...,n$. This is because
        \begin{align*}
            \left\|  f^{\rm Peri}(\bfX^{(1)}) - f^{\rm Peri}(\bfX^{(2)}) \right\|_F &= \sqrt{\sum_{j=1}^n \left\|  f_j^{\rm Peri}(\bfX^{(1)}) - f_j^{\rm Peri}(\bfX^{(1)}) \right\|_2^2} \\
            &= \sqrt{\sum_{j=1}^n \left\| \bfgamma^{\rm out} \odot \hat{f}_j(\bfX^{(1)}) \cancel{+ \bfbeta^{\rm out}} -\bfgamma^{\rm out} \odot \hat{f}_j(\bfX^{(2)}) \cancel{- \bfbeta^{\rm out}} \right\|_2^2} \\ 
            & \leq \sqrt{ \sum_{j=1}^n \left( \left\| \bfgamma^{\rm out} \odot \hat{f}_j(\bfX^{(1)}) \right\|_2  + \left\| \bfgamma^{\rm out} \odot \hat{f}_j(\bfX^{(2)}) \right\|_2 \right)^2} \\
            &\leq \gamma_{\rm max} \sqrt{ \sum_{j=1}^n \left( \left\| \hat{f}_j(\bfX^{(1)}) \right\|_2 + \left\| \hat{f}_j(\bfX^{(2)}) \right\|_2 \right)^2 }\\
            &\leq \gamma_{\rm max} \sqrt{ \sum_{j=1}^n \left( 2 \sqrt{d} \right)^2 }, \quad \text{by (\ref{eq:col_f_bound})},\\
            &= 2 \sqrt{nd} \gamma_{\rm max} . 
        \end{align*}
    \end{itemize}
    Now, we derive the results as 
    \begin{align}
        & \| \bfX^{(1)}_D - \bfX^{(2)}_D \|_F \\
        & \leq \left\|  \bfX_{0}^{(1)} + \sum_{i=0}^{D-1} \left( f_{\rm attn, i}^{\rm Peri}(\bfX_{i}^{(1)}) + f_{\rm ffn, i}^{\rm Peri}(\bfU_{i}^{(1)})\right) -  \bfX_{0}^{(2)} - \sum_{i=0}^{D-1} \left( f_{\rm attn, i}^{\rm Peri}(\bfX_{i}^{(2)}) + f_{\rm ffn, i}^{\rm Peri}(\bfU_{i}^{(2)})\right)\right\|_F, \quad \text{by (\ref{eq:expanded recursion})}\\
        & \leq \| \bfX_{0}^{(1)} - \bfX_{0}^{(2)} \|_F + \sum_{i=0}^{D-1} \left\|  \left( f_{\rm attn, i}^{\rm Peri}(\bfX_{i}^{(1)}) - f_{\rm attn, i}^{\rm Peri}(\bfX_{i}^{(2)})\right) \right\|_F +\sum_{i=0}^{D-1} \left\|  \left(  f_{\rm ffn, i}^{\rm Peri}(\bfU_{i}^{(1)})- f_{\rm ffn, i}^{\rm Peri}(\bfU_{i}^{(2)})\right) \right\|_F\\
        &\leq \| \bfX_{0}^{(1)} - \bfX_{0}^{(2)} \|_F + \sum_{i=0}^{D-1} (2 \sqrt{nd} \gamma_{\rm max}) + \sum_{i=0}^{D-1} (2 \sqrt{nd} \gamma_{\rm max}), \quad \text{by (\ref{eq:diff_term_bound})},\\
        &= \| \bfX_{0}^{(1)} - \bfX_{0}^{(2)} \|_F + 4 D \sqrt{nd} \gamma_{\rm max} \label{eq:datawise_sensitivity}
    \end{align}
    By properties of Wasserstein distance, we have
    \begin{align*}
        W_p^p(\mu_D,\nu_D) &\leq \int \| \bfX^{(1)}_D - \bfX^{(2)}_D \|_p^p \; d\pi_0(\bfX^{(1)}_0, \bfX^{(2)}_0), \\
        \intertext{where $\pi_0$ denotes the optimal transport plan between $\mu_0$ and $\nu_0$ under the $W_p$ distance, and $\bfX^{(1)}_D$ is the corresponding terminal state for input $\bfX^{(1)}_0$ under the Peri-LN Transformer, and $\|\bfX\|_p:= \left( \sum_{i,j} |X_{ij}|^p \right)^{1/p}$ is computed computed as the vector $p$-norm of all entries in $\bfX$, rather than the matrix $p$-norm,} \\
        &\leq \int \left( \hat{C}(p)\| \bfX^{(1)}_0-\bfX^{(2)}_0\|_p + 4 D \sqrt{nd} \gamma_{\rm max} \right)^p \; d\pi_0(\bfX^{(1)}_0, \bfX^{(2)}_0), \\\intertext{by~\eqref{eq:datawise_sensitivity} and the equivalence of norms, for some $\hat{C}(p)>0$ depending on $p>1$,} \\
        &\leq 2^{p-1} \int \left( \hat{C}^p(p)\| \bfX^{(1)}_0-\bfX^{(2)}_0\|_p^p + (4 D \sqrt{nd} \gamma_{\rm max})^p \right) \; d\pi_0(\bfX^{(1)}_0, \bfX^{(2)}_0), \\
        \intertext{by convexity,}\\
        &= 2^{p-1} \left( \hat{C}^p(p) W_p^p(\mu_0,\nu_0) + (4 D \sqrt{nd} \gamma_{\rm max})^p \right).
    \end{align*}
    Now, taking the $p$-th root on both sides, we have
    \begin{align*}
        W_p(\mu_D,\nu_D) &\leq 2^{\frac{p-1}{p}} \left( \hat{C}^p(p) W_p^p(\mu_0,\nu_0) + (4 D \sqrt{nd} \gamma_{\rm max})^p \right)^{\frac{1}{p}} \\
        &\leq 2^{\frac{p-1}{p}} \left( \hat{C}(p) W_p(\mu_0,\nu_0) + 4 D \sqrt{nd} \gamma_{\rm max} \right), \\
        \intertext{by the Minkowski inequality, and we obtain the desired result.}
    \end{align*}
\end{proof} 

\paragraph{Bounds on Generalization} We can use the result in~\Cref{thm:uq} and techniques from distributionally robust optimization (DRO) to derive generalization bounds. We first specify the setup and notations.

Suppose the Peri-LN Transformer is trained on the training data $\bfX^{(1)}, \dots, \bfX^{(N)}$ that form the empirical training distribution $\hat{\mu}_N$. Denote $\mathcal{W}_{1,r}(\hat{\mu}_N)$ to be the $1$-Wasserstein ball of radius $r$ around the empirical training distribution $\hat{\mu}_N$, i.e., $\mathcal{W}_r(\hat{\mu}_N):= \{\rho \in \mathcal{P}(\Omega) : W_{1,r}(\rho,\hat{\mu}_N) \le r \}$. We denote by $\bfT$ the forward mapping of the Peri-LN Transformer from the input to the terminal hidden states. We make the following two assumptions: 1.$\Omega$ is a compact domain, and 2. the training loss $G$ in~\eqref{eq:discrete_time_OC_obj} is Lipschitz continuous with Lipschitz constant $L$. For common applications, the training loss $G$ is given by the composition of softmax function and cross-entropy loss, which is smooth and thus Lipschitz continuous in a compact domain, see~\cite{Kan2024LSEMINK}.

Our goal is to derive bounds of the model performance on distributions $\nu \in \mathcal{W}_{1,r}(\hat{\mu}_N)$, where the data can be in-distribution or out-of-distribution of the training data. By Kantorovich duality, we have the variational formula 
\begin{align*}
    W_1(\hat{\mu}_N, \nu) = \sup_{\varphi \in \text{Lip}_1(\Omega)} \{\mathbb{E}_{\hat{\mu}_N}[\varphi(\bfX)] - \mathbb{E}_{\nu}[\varphi(\bfX)] \},
\end{align*}
where $\text{Lip}_1(\Omega)$ denotes the set of Lipschitz functions on $\Omega$. By~\Cref{thm:uq}, we have
\begin{align}
    \mathbb{E}_{\mathbf{T}_\sharp \nu}[ G(\bfX_D,\bfy)] - \mathbb{E}_{\mathbf{T}_\sharp\hat{\mu}_N}[ G(\bfX_D,\bfy)] &\leq  L \cdot W_1(\mathbf{T}_\sharp \hat{\mu}_N, \mathbf{T}_\sharp \nu) \\
    &\leq L \cdot \left( \hat{C}(1) W_1(\mu_0,\nu_0) + 4 D \sqrt{nd} \gamma_{\rm max} \right) \\
    &\leq L \cdot \left( \hat{C}(1) r + 4 D \sqrt{nd} \gamma_{\rm max} \right).
\end{align}
By rearranging the terms, we have
\begin{equation}\label{eq:DRO_deter_bound}
    \mathbb{E}_\nu[ G(\bfT(\bfX_0),\bfy)] \le \mathbb{E}_{\hat{\mu}_N}[ G(\bfT(\bfX_0),\bfy)] +  L \cdot \left( \hat{C}(1) r + 4 D \sqrt{nd} \gamma_{\rm max} \right).
\end{equation}
Thus, we obtain a deterministic bound for the loss on distributions that is within a Wasserstein ball from the empirical training distribution.

We can further turn this deterministic bound into a statistical bound, where we ensure that the true distribution is within the Wasserstein ball with high probability. In particular, let $\mu$ be the true distribution that generates the empirical training distribution $\hat{\mu}_N$. Under standard assumptions, for instance, sub-Gaussian tails or bounded domains, we have $\mu \in \mathcal{W}_r(\hat{\mu}_N)$ with high probability~\cite{FournierGuillin2015}: there exists $\bar{C} > 0$ such that with probability at least $1-\delta$, 
\begin{align}\label{eq:sample_complexity}
    W_1(\mu, \hat{\mu}_N) \le \bar{C} N^{-\frac{1}{nd}} + \sqrt{\frac{\log(1/\delta)}{N}} = r(N,\delta)\, .
\end{align}
Here, the radius of the Wasserstein ball can be chosen based on the number of training data $N$ and probability $1-\delta$. Combining the deterministic bound~\eqref{eq:DRO_deter_bound} and the statistical guarantee~\eqref{eq:sample_complexity}, we obtain, with probability at least $1-\delta$, 
\begin{equation}
    \mathbb{E}_\nu[ G(\bfT(\bfX_0),\bfy)] \le \mathbb{E}_{\hat{\mu}_N}[ G(\bfT(\bfX_0),\bfy)] +  L \cdot \left( \hat{C}(1) r(N,\delta) + 4 D \sqrt{nd} \gamma_{\rm max} \right).
\end{equation}
We remark that this result is a non-asymptotic generalization bound in term of Wasserstein distance.

\section{Backward Stability of Transformers}
We restate and prove~\Cref{thm:gradient_explosion,thm:gradient_stability} together.

\settheoremnumber{thm:gradient_explosion}
\begin{proposition}
Under Pre-LN, the sensitivity $\nabla_{\bfX_{j-1}} f^{\rm Pre}(\bfX_{j-1}; \bftheta_{j-1})$ grows proportionally with the activations. 
\end{proposition}
\restoretheoremnumber

\settheoremnumber{thm:gradient_stability}
\begin{proposition}
 Under Peri-LN, the sensitivity $\nabla_{\bfX_{j-1}} f^{\rm Peri}(\bfX_{j-1}; \bftheta_{j-1})$ is \textit{invariant} to the magnitude of the activation. 
\end{proposition}
\restoretheoremnumber

Both propositions can be verified by inspecting the invariance of the layer normalization operation, or by directly deriving the gradients, which we perform in the following.

For clarity in this section, we first explicitly define a few additional notations for Transformers under Pre-LN and Peri-LN. This will allow us to derive the gradients in a fully explicit manner.

Under Pre-LN, the $i$th Transformer block reads $\bfX^{{\rm Pre}}_{i+1}=f^{\rm Pre}(\bfX^{{\rm Pre}}_i; \bftheta_i)$, where
\begin{align}
     \bfX^{{\rm Pre}, 1}_{i} &= {\rm LN}(\bfX^{{\rm Pre}}_i; \bfgamma_i^{\rm attn}, \bfbeta_i^{\rm attn}) \\
     \bfX^{{\rm Pre}, 2}_{i} &= f_{\rm attn}(\bfX_i^{{\rm Pre}, 1}; \bftheta^{\rm attn}_i) = \sum_{h=1}^H  \bfW_i^h \bfV_i^h \bfX_i^{{\rm Pre}, 1} \, {\rm softmax} \left( \frac{(\bfK_i^h \bfX_i^{{\rm Pre}, 1})^\top \bfQ_i^h \bfX_{i}^{{\rm Pre}, 1}}{\sqrt{k}} \right) \label{eq:attn_pre_activation} \\
     \bfX^{{\rm Pre}, 3}_{i} &= \bfX^{\rm Pre}_i + \bfX^{{\rm Pre}, 2}_{i} \\
     \bfX^{{\rm Pre}, 4}_{i} &= {\rm LN}(\bfX^{{\rm Pre}, 3}_{i}; \bfgamma_i^{\rm ffn}, \bfbeta_i^{\rm ffn}) \\
     \bfX^{{\rm Pre}, 5}_{i} &=  f_{\rm ffn}(\bfX^{{\rm Pre}, 4}_{i};\bftheta^{\rm ffn}_{i}) = \bfW_{i}^{(2)} \phi \left( \bfW_{i}^{(1)}\bfX^{{\rm Pre}, 4}_{i} \right) \label{eq:ffn_pre_activation} \\
     \bfX^{{\rm Pre}}_{i+1} &= \bfX^{{\rm Pre}, 3}_{i} + \bfX^{{\rm Pre}, 5}_{i},
\end{align}
for $i=0,1,...,D-1$. Here $D$ is the total number of Transformer blocks,  $\bftheta_i$ collectively denotes $(\bftheta^{\rm attn}_i, \bftheta^{\rm ffn}_{i})$ and $\phi$ denotes an activation function. For simplicity and WLOG, we omit the bias term in the feedforward sublayer, as it can be absorbed into the weight matrix by augmenting the input with a constant one. 

Under Peri-LN, the $i$th Transformer block reads $\bfX^{{\rm Peri}}_{i+1}=f^{\rm Peri}(\bfX^{{\rm Peri}}_i; \bftheta_i)$, where
\begin{align}
     \bfX^{{\rm Peri}, 1}_{i} &= {\rm LN}(\bfX^{{\rm Peri}}_i; \bfgamma_{{\rm in}, i}^{\rm attn}, \bfbeta_{{\rm in}, i}^{\rm attn}) \\
     \bfX^{{\rm Peri}, 2}_{i} &= f_{\rm attn}(\bfX_i^{{\rm Peri}, 1}; \bftheta^{\rm attn}_i) = \sum_{h=1}^H  \bfW_i^h \bfV_i^h \bfX_i^{{\rm Peri}, 1} \, {\rm softmax} \left( \frac{(\bfK_i^h \bfX_i^{{\rm Peri}, 1})^\top \bfQ_i^h \bfX_{i}^{{\rm Peri}, 1}}{\sqrt{k}} \right) \label{eq:Peri_2_appendix} \\
     \bfX^{{\rm Peri}, 3}_{i} &= {\rm LN}(\bfX^{{\rm Peri}, 2}_i; \bfgamma_{{\rm out}, i}^{\rm attn}, \bfbeta_{{\rm out}, i}^{\rm attn})\\
     \bfX^{{\rm Peri}, 4}_{i} &= \bfX^{\rm Peri}_i + \bfX^{{\rm Peri}, 3}_{i} \\
     \bfX^{{\rm Peri}, 5}_{i} &= {\rm LN}(\bfX^{{\rm Peri}, 4}_{i}; \bfgamma_{{\rm in}, i}^{\rm ffn}, \bfbeta_{{\rm in}, i}^{\rm ffn}) \\
     \bfX^{{\rm Peri}, 6}_{i} &=  f_{\rm ffn}(\bfX^{{\rm Peri}, 5}_{i};\bftheta^{\rm ffn}_{i}) = \bfW_{i}^{(2)} \phi \left( \bfW_{i}^{(1)}\bfX^{{\rm Peri}, 5}_{i}  \right)  \label{eq:Peri_6_appendix}\\
     \bfX^{{\rm Peri}, 7}_{i} &= {\rm LN}(\bfX^{{\rm Peri}, 6}_i; \bfgamma_{{\rm out}, i}^{\rm ffn}, \bfbeta_{{\rm out}, i}^{\rm ffn})\\
     \bfX^{{\rm Peri}}_{i+1} &= \bfX^{{\rm Peri}, 4}_{i} + \bfX^{{\rm Peri}, 7}_{i}.
\end{align}

The gradient with respect to $\bftheta_i$, the weights of the $i$th layer, is given by
\begin{align}\label{eq:gradient_product_appendix}
\begin{split}
    \nabla_{\bftheta_l} G(\bfX_D) &= \nabla_{\bftheta_l} \bfX_{l+1} \left( \prod_{i=l+1}^{D-1} \underbrace{\nabla_{\bfX_{i}} \bfX_{i+1}}_{\text{local sensitivity at $l$th block}} \right) \nabla_{\bfX_D} G(\bfX_D). 
\end{split}
\end{align}

Here, we derive and analyze the local sensitivity. For clarity of presentation, we denote $\tilde{\bfx}_i = {\rm vec}(\bfX_i) \in \mathbb{R}^{nd}$ as the vectorized hidden states. We also denote $\bfx_{i,j} \in \mathbb{R}^{d}$ as the $j$th column of $\bfX_i$. Under this vectorization, the local sensitivity term is given as $\nabla_{\tilde{\bfx}_{i}} \tilde{\bfx}_{i+1} \in \mathbb{R}^{nd \times nd}$.

In the following, we show that under Pre-LN, exploding activation ($\tilde{\bfx}_i^{{\rm Pre}, 2}$ in (\ref{eq:attn_pre_activation}) or $\tilde{\bfx}_i^{{\rm Pre}, 5}$ in (\ref{eq:ffn_pre_activation}) explodes in magnitude) causes the local sensitivity at $i$th block to explode.

Under Pre-LN, the local sensitivity term is given by
\begin{align}
    \nabla_{\tilde{\bfx}_{i}^{\rm Pre}} \tilde{\bfx}_{i+1}^{\rm Pre} &= \underbrace{\nabla_{\tilde{\bfx}_i^{{\rm Pre}}} \tilde{\bfx}_i^{{\rm Pre}, 3}}_{\text{sensitivity of self-attention sublayer}}  \underbrace{\nabla_{\tilde{\bfx}_i^{{\rm Pre}, 3}} \tilde{\bfx}_{i+1}^{{\rm Pre}}}_{\text{sensitivity of feedforward sublayer}}.
\end{align}

Here, the sensitivity of the self-attention sublayer is given by
\begin{align}
    \nabla_{\tilde{\bfx}_i^{{\rm Pre}}} \tilde{\bfx}_i^{{\rm Pre}, 3} &= \bfI +  \nabla_{\tilde{\bfx}_i^{{\rm Pre}}} \tilde{\bfx}_i^{{\rm Pre}, 2} \\
    &= \bfI + \begin{bmatrix}
\nabla {\rm LN}(\bfx^{{\rm Pre}}_{i, 1}) &        &        \\
                 & \ddots &        \\
                 &        & \nabla {\rm LN}(\bfx^{{\rm Pre}}_{i, n})
\end{bmatrix} \\ &
\begin{bmatrix}
\nabla_{\bfx^{{\rm Pre},1}_{i, 1}} [f_{\rm attn}(\tilde{\bfx}^{{\rm Pre},1}_{i})]_1 &  \nabla_{\bfx^{{\rm Pre},1}_{i, 1}} [f_{\rm attn}(\tilde{\bfx}^{{\rm Pre},1}_{i})]_2  & ...    &    \nabla_{\bfx^{{\rm Pre},1}_{i, 1}} [f_{\rm attn}(\tilde{\bfx}^{{\rm Pre},1}_{i})]_n    \\
\nabla_{\bfx^{{\rm Pre},1}_{i, 2}} [f_{\rm attn}(\tilde{\bfx}^{{\rm Pre},1}_{i})]_1 &  \nabla_{\bfx^{{\rm Pre},1}_{i, 2}} [f_{\rm attn}(\tilde{\bfx}^{{\rm Pre},1}_{i})]_2  & ...    &    \nabla_{\bfx^{{\rm Pre},1}_{i, 2}} [f_{\rm attn}(\tilde{\bfx}^{{\rm Pre},1}_{i})]_n    \\
     \vdots        &  \vdots  & \ddots &    \vdots    \\
\nabla_{\bfx^{{\rm Pre},1}_{i, n}} [f_{\rm attn}(\tilde{\bfx}^{{\rm Pre},1}_{i})]_1 &  \nabla_{\bfx^{{\rm Pre},1}_{i, n}} [f_{\rm attn}(\tilde{\bfx}^{{\rm Pre},1}_{i})]_2  & ...    &    \nabla_{\bfx^{{\rm Pre},1}_{i, n}} [f_{\rm attn}(\tilde{\bfx}^{{\rm Pre},1}_{i})]_n,   
\end{bmatrix}
\end{align}
where $\nabla_{\bfx^{{\rm Pre},1}_{i, l}} [f_{\rm attn}(\tilde{\bfx}^{{\rm Pre},1}_{i})]_j$'s are given in (\ref{eq:grad_attn}).

In the activation of the self-attention sublayer (\ref{eq:attn_pre_activation}), the input $\tilde{\bfx}^{\rm Pre, 1}_i$ to the sublayer is normalized and bounded in magnitude (Lemma~\ref{lemma:ellipsoid}). Also the output of softmax is bounded (each column sums to one). Thus, when large activation occurs ($\tilde{\bfx}^{\rm Pre, 2}_i$ is large), at least one of $\bfW_{i}^{h}$ or $\bfV_{i}^{h}$ must be large. Since $\nabla {\rm LN}(\bfx^{{\rm Pre}}_{i, j})$'s are independent of $\bfW_{i}^{h}$ and $\bfV_{i}^{h}$, and $\nabla_{\bfx^{{\rm Pre},1}_{i, l}} [f_{\rm attn}(\tilde{\bfx}^{{\rm Pre},1}_{i})]_j$ scales linearly with $\bfW_{i}^{h}$ and $\bfV_{i}^{h}$ (Proposition~\ref{prop:grad_attn}), the sensitivity $\nabla_{\tilde{\bfx}_i^{{\rm Pre}}} \tilde{\bfx}_i^{{\rm Pre}, 3}$ must have large magnitude.

Moreover, the sensitivity of the feedforward sublayer is given by
\begin{align}
   \nabla_{\tilde{\bfx}_i^{{\rm Pre}, 3}} \tilde{\bfx}_{i+1}^{{\rm Pre}} &= \bfI + \nabla_{\tilde{\bfx}_i^{{\rm Pre}, 3}} \tilde{\bfx}_{i+1}^{{\rm Pre, 5}} \\
   &= \bfI 
   +  \begin{bmatrix}
\nabla {\rm LN}(\bfx^{{\rm Pre}, 3}_{i, 1}) &        &        \\
                 & \ddots &        \\
                 &        & \nabla {\rm LN}(\bfx^{{\rm Pre}, 3}_{i, n})
\end{bmatrix} \\
&\begin{bmatrix}
\bfW_{i}^{(1)} &        &        \\
                 & \ddots &        \\
                 &        & \bfW_{i}^{(1)}
\end{bmatrix}^\top
   \begin{bmatrix}
\bfh'_1 &        &        \\
                 & \ddots &        \\
                 &        & \bfh'_n
\end{bmatrix}
\begin{bmatrix}
\bfW_{i}^{(2)} &        &        \\
                 & \ddots &        \\
                 &        & \bfW_{i}^{(2)}
\end{bmatrix}^\top,
\end{align}
where
\begin{equation}
\nabla {\rm LN}(\bfx; \bfgamma, \bfbeta) = \frac{{\rm diag}(\bfgamma)}{\sigma(\bfx)} - \frac{1}{d} \frac{(\bfx-\mu(\bfx))(\bfgamma \odot (\bfx-\mu(\bfx)))^\top}{\sigma(\bfx)^3}
\end{equation}
by Proposition~\ref{prop:grad_LN}, and $\bfh'_j = {\rm diag} (\phi' ( \bfW_{i}^{(1)}\bfx^{{\rm Pre}, 4}_{i,j}  ))
$. We remark that $\phi'$ is the derivative of the activation function, thus the entries of $\bfh'_j$ are bounded by a constant in most cases (e.g., at most 1 for sigmoid, tanh, and ReLU).

When large activation occurs in the feed-forward sublayer (\ref{eq:ffn_pre_activation}),  i.e., $\tilde{\bfx}^{\rm Pre, 5}_i$ is large, since the input $\tilde{\bfx}^{\rm Pre, 4}_i$ to the sublayer is normalized and bounded in magnitude (Lemma~\ref{lemma:ellipsoid}), at least one of $\bfW_{i}^{(1)}$ or $\bfW_{i}^{(2)}$ must be large. Since $\bfh'_j$'s have bounded entries, and $\nabla {\rm LN}(\bfx^{{\rm Pre}, 3}_{i, j})$'s are independent of $\bfW_{i}^{(1)}$ and $\bfW_{i}^{(2)}$, $\nabla_{\tilde{\bfx}_i^{{\rm Pre}, 3}} \tilde{\bfx}_{i+1}^{{\rm Pre}}$ must have large magnitude.

Next, we show that under Peri-LN, the local sensitivity at $i$th block is invariant to re-scaling of activations ($\tilde{\bfx}_i^{{\rm Peri}, 2}$ in (\ref{eq:Peri_2_appendix}) or $\tilde{\bfx}_i^{{\rm Peri}, 6}$ in (\ref{eq:Peri_6_appendix})). This implies that even when activation explodes in magnitude, the local sensitivity stay at its nominal magnitude.

For Peri-LN, the local sensitivity term is given by 
\begin{align}
    \nabla_{\tilde{\bfx}_{i}^{\rm Peri}} \tilde{\bfx}_{i+1}^{\rm Peri} &= \underbrace{\nabla_{\tilde{\bfx}_i^{{\rm Peri}}} \tilde{\bfx}_i^{{\rm Peri}, 4}}_{\text{sensitivity of self-attention sublayer}}  \underbrace{\nabla_{\tilde{\bfx}_i^{{\rm Peri}, 4}} \tilde{\bfx}_{i+1}^{{\rm Peri}}}_{\text{sensitivity of feedforward sublayer}}.
\end{align}

Here, the sensitivity of the self-attention sublayer is given by
\begin{align}
    \nabla_{\tilde{\bfx}_i^{{\rm Peri}}} \tilde{\bfx}_i^{{\rm Peri}, 4} &= \bfI +  \nabla_{\tilde{\bfx}_i^{{\rm Peri}}} \tilde{\bfx}_i^{{\rm Peri}, 3} \\
    &= \bfI + \begin{bmatrix}
\nabla {\rm LN}(\bfx^{{\rm Peri}}_{i, 1}) &        &        \\
                 & \ddots &        \\
                 &        & \nabla {\rm LN}(\bfx^{{\rm Peri}}_{i, n})
\end{bmatrix} \\ &
\begin{bmatrix}
\nabla_{\bfx^{{\rm Peri},1}_{i, 1}} [f_{\rm attn}(\tilde{\bfx}^{{\rm Peri},1}_{i})]_1 &  \nabla_{\bfx^{{\rm Peri},1}_{i, 1}} [f_{\rm attn}(\tilde{\bfx}^{{\rm Peri},1}_{i})]_2  & ...    &    \nabla_{\bfx^{{\rm Peri},1}_{i, 1}} [f_{\rm attn}(\tilde{\bfx}^{{\rm Peri},1}_{i})]_n    \\
\nabla_{\bfx^{{\rm Peri},1}_{i, 2}} [f_{\rm attn}(\tilde{\bfx}^{{\rm Peri},1}_{i})]_1 &  \nabla_{\bfx^{{\rm Peri},1}_{i, 2}} [f_{\rm attn}(\tilde{\bfx}^{{\rm Peri},1}_{i})]_2  & ...    &    \nabla_{\bfx^{{\rm Peri},1}_{i, 2}} [f_{\rm attn}(\tilde{\bfx}^{{\rm Peri},1}_{i})]_n    \\
     \vdots        &  \vdots  & \ddots &    \vdots    \\
\nabla_{\bfx^{{\rm Peri},1}_{i, n}} [f_{\rm attn}(\tilde{\bfx}^{{\rm Peri},1}_{i})]_1 &  \nabla_{\bfx^{{\rm Peri},1}_{i, n}} [f_{\rm attn}(\tilde{\bfx}^{{\rm Peri},1}_{i})]_2  & ...    &    \nabla_{\bfx^{{\rm Peri},1}_{i, n}} [f_{\rm attn}(\tilde{\bfx}^{{\rm Peri},1}_{i})]_n,   
\end{bmatrix}\\ 
& \begin{bmatrix}
\nabla {\rm LN}(\bfx^{{\rm Peri}, 2}_{i, 1}) &        &        \\
                 & \ddots &        \\
                 &        & \nabla {\rm LN}(\bfx^{{\rm Peri}, 2}_{i, n})
\end{bmatrix}
\end{align}
where $\nabla_{\bfx^{{\rm Peri},1}_{i, l}} [f_{\rm attn}(\tilde{\bfx}^{{\rm Peri},1}_{i})]_j$'s are given in (\ref{eq:grad_attn}).

We first remark that this sensitivity for the self-attention sublayer is invariant under re-scaling of $\bfW_{i}^{h}$ and $\bfV_{i}^{h}$. Note that $\bfx^{{\rm Peri}}_{i, j}$'s are independent of $\bfW_{i}^{h}$ and $\bfV_{i}^{h}$. If we multiply $\bfW_{i}^{h}$ and $\bfV_{i}^{h}$ by positive constants $c_1$ and $c_2$, respectively, then ${\bfx}^{{\rm Peri}, 2}_{i,j}$'s become $c_1 c_2 {\bfx}^{{\rm Peri}, 2}_{i,j}$ by~(\ref{eq:Peri_2_appendix}), and $\nabla_{\bfx^{{\rm Peri},1}_{i, l}} [f_{\rm attn}(\tilde{\bfx}^{{\rm Peri},1}_{i})]_j$'s become $c_1 c_2 \nabla_{\bfx^{{\rm Peri},1}_{i, l}} [f_{\rm attn}(\tilde{\bfx}^{{\rm Peri},1}_{i})]_j$ by Proposition~\ref{prop:grad_attn}, and we have
\begin{align}
    \nabla_{\tilde{\bfx}_i^{{\rm Peri}}} \tilde{\bfx}_i^{{\rm Peri}, 4} 
    &= \bfI + \begin{bmatrix}
\nabla {\rm LN}(\bfx^{{\rm Peri}}_{i, 1}) &        &        \\
                 & \ddots &        \\
                 &        & \nabla {\rm LN}(\bfx^{{\rm Peri}}_{i, n})
\end{bmatrix} \\ &
\begin{bmatrix}
c_1 c_2 \nabla_{\bfx^{{\rm Peri},1}_{i, 1}} [f_{\rm attn}(\tilde{\bfx}^{{\rm Peri},1}_{i})]_1 & c_1 c_2 \nabla_{\bfx^{{\rm Peri},1}_{i, 1}} [f_{\rm attn}(\tilde{\bfx}^{{\rm Peri},1}_{i})]_2  & ...    &  c_1 c_2  \nabla_{\bfx^{{\rm Peri},1}_{i, 1}} [f_{\rm attn}(\tilde{\bfx}^{{\rm Peri},1}_{i})]_n    \\
c_1 c_2 \nabla_{\bfx^{{\rm Peri},1}_{i, 2}} [f_{\rm attn}(\tilde{\bfx}^{{\rm Peri},1}_{i})]_1 & c_1 c_2 \nabla_{\bfx^{{\rm Peri},1}_{i, 2}} [f_{\rm attn}(\tilde{\bfx}^{{\rm Peri},1}_{i})]_2  & ...    &  c_1 c_2  \nabla_{\bfx^{{\rm Peri},1}_{i, 2}} [f_{\rm attn}(\tilde{\bfx}^{{\rm Peri},1}_{i})]_n    \\
     \vdots        &  \vdots  & \ddots &    \vdots    \\
c_1 c_2 \nabla_{\bfx^{{\rm Peri},1}_{i, n}} [f_{\rm attn}(\tilde{\bfx}^{{\rm Peri},1}_{i})]_1 & c_1 c_2 \nabla_{\bfx^{{\rm Peri},1}_{i, n}} [f_{\rm attn}(\tilde{\bfx}^{{\rm Peri},1}_{i})]_2  & ...    &  c_1 c_2  \nabla_{\bfx^{{\rm Peri},1}_{i, n}} [f_{\rm attn}(\tilde{\bfx}^{{\rm Peri},1}_{i})]_n,   
\end{bmatrix}\\ 
& \begin{bmatrix}
\nabla {\rm LN}(c_1 c_2 \bfx^{{\rm Peri}, 2}_{i, 1}) &        &        \\
                 & \ddots &        \\
                 &        & \nabla {\rm LN}( c_1 c_2\bfx^{{\rm Peri}, 2}_{i, n})
\end{bmatrix} \\
& \text{using Proposition~\ref{prop:grad_LN}, $\nabla {\rm LN}(c_1 c_2 \bfx^{{\rm Peri}, 2}_{i, j})=\frac{1}{c_1 c_2} \nabla {\rm LN}(\bfx^{{\rm Peri}, 2}_{i, j})$, we have} \\
&= \bfI + \begin{bmatrix}
\nabla {\rm LN}(\bfx^{{\rm Peri}}_{i, 1}) &        &        \\
                 & \ddots &        \\
                 &        & \nabla {\rm LN}(\bfx^{{\rm Peri}}_{i, n})
\end{bmatrix} \\ 
& \cancel{(c_1 c_2)} \begin{bmatrix}
 \nabla_{\bfx^{{\rm Peri},1}_{i, 1}} [f_{\rm attn}(\tilde{\bfx}^{{\rm Peri},1}_{i})]_1 &  \nabla_{\bfx^{{\rm Peri},1}_{i, 1}} [f_{\rm attn}(\tilde{\bfx}^{{\rm Peri},1}_{i})]_2  & ...    &   \nabla_{\bfx^{{\rm Peri},1}_{i, 1}} [f_{\rm attn}(\tilde{\bfx}^{{\rm Peri},1}_{i})]_n    \\
 \nabla_{\bfx^{{\rm Peri},1}_{i, 2}} [f_{\rm attn}(\tilde{\bfx}^{{\rm Peri},1}_{i})]_1 &  \nabla_{\bfx^{{\rm Peri},1}_{i, 2}} [f_{\rm attn}(\tilde{\bfx}^{{\rm Peri},1}_{i})]_2  & ...    &    \nabla_{\bfx^{{\rm Peri},1}_{i, 2}} [f_{\rm attn}(\tilde{\bfx}^{{\rm Peri},1}_{i})]_n    \\
     \vdots        &  \vdots  & \ddots &    \vdots    \\
\nabla_{\bfx^{{\rm Peri},1}_{i, n}} [f_{\rm attn}(\tilde{\bfx}^{{\rm Peri},1}_{i})]_1 &  \nabla_{\bfx^{{\rm Peri},1}_{i, n}} [f_{\rm attn}(\tilde{\bfx}^{{\rm Peri},1}_{i})]_2  & ...    &   \nabla_{\bfx^{{\rm Peri},1}_{i, n}} [f_{\rm attn}(\tilde{\bfx}^{{\rm Peri},1}_{i})]_n,   
\end{bmatrix}\\ 
& \cancel{(\frac{1}{c_1 c_2})} \begin{bmatrix}
\nabla {\rm LN}(\bfx^{{\rm Peri}, 2}_{i, 1}) &        &        \\
                 & \ddots &        \\
                 &        & \nabla {\rm LN}(\bfx^{{\rm Peri}, 2}_{i, n})
\end{bmatrix},
\end{align}
which remains the same and verifies the re-scaling invariance.

In the activation of the self-attention sublayer (\ref{eq:Peri_2_appendix}), the input $\tilde{\bfx}^{\rm Peri, 1}_i$ to the sublayer is normalized and bounded in magnitude (Lemma~\ref{lemma:ellipsoid}). Also the output of softmax is bounded (each column sums to one). Thus, when large activation occurs ($\tilde{\bfx}^{\rm Peri, 2}_i$ is large), at least one of $\bfW_{i}^{h}$ or $\bfV_{i}^{h}$ must be large. However, since the sensitivity of self-attention sublayer $\nabla_{\tilde{\bfx}_i^{{\rm Peri}}} \tilde{\bfx}_i^{{\rm Peri}, 4}$ is invariant under re-scaling of $\bfW_{i}^{h}$ and $\bfV_{i}^{h}$, it stays at its normal magnitude.

Next, sensitivity of feedforward sublayer is given by
\begin{align}
   \nabla_{\tilde{\bfx}_i^{{\rm Peri}, 4}} \tilde{\bfx}_{i+1}^{{\rm Peri}} &= \bfI + \nabla_{\tilde{\bfx}_i^{{\rm Peri}, 4}} \tilde{\bfx}_{i+1}^{{\rm Peri, 7}} \\
   &= \bfI 
   +  \begin{bmatrix}
\nabla {\rm LN}(\bfx^{{\rm Peri}, 4}_{i, 1}) &        &        \\
                 & \ddots &        \\
                 &        & \nabla {\rm LN}(\bfx^{{\rm Peri}, 4}_{i, n})
\end{bmatrix} 
\begin{bmatrix}
\bfW_{i}^{(1)} &        &        \\
                 & \ddots &        \\
                 &        & \bfW_{i}^{(1)}
\end{bmatrix}^\top \\
   & \begin{bmatrix}
\bfh'_1 &        &        \\
                 & \ddots &        \\
                 &        & \bfh'_n
\end{bmatrix}
\begin{bmatrix}
\bfW_{i}^{(2)} &        &        \\
                 & \ddots &        \\
                 &        & \bfW_{i}^{(2)}
\end{bmatrix}^\top \begin{bmatrix}
\nabla {\rm LN}(\bfx^{{\rm Peri}, 6}_{i, 1}) &        &        \\
                 & \ddots &        \\
                 &        & \nabla {\rm LN}(\bfx^{{\rm Peri}, 6}_{i, n})
\end{bmatrix},
\end{align}
where
\begin{equation}
\nabla {\rm LN}(\bfx; \bfgamma, \bfbeta) = \frac{{\rm diag}(\bfgamma)}{\sigma(\bfx)} - \frac{1}{d} \frac{(\bfx-\mu(\bfx))(\bfgamma \odot (\bfx-\mu(\bfx)))^\top}{\sigma(\bfx)^3}
\end{equation}
by Proposition~\ref{prop:grad_LN}, and $\bfh'_j = {\rm diag} (\phi' ( \bfW_{i}^{(1)}\bfx^{{\rm Peri}, 5}_{i,j} ))
$. We remark that $\phi'$ is the derivative of the activation function, thus the entries of $\bfh'_j$ are bounded by a constant in most cases (e.g., at most 1 for sigmoid, tanh, and ReLU).

We first note that this sensitivity for the feed-forward sublayer is invariant under re-scaling of $\bfW_{i}^{(1)}$ and $\bfW_{i}^{(2)}$. Specifically, if we multiply $\bfW_{i}^{(1)}$ and $\bfW_{i}^{(2)}$ by positive constants $c_1$ and $c_2$, respectively, then $\tilde{\bfx}^{{\rm Peri}, 6}_i$ becomes $c_1 c_2 \tilde{\bfx}^{{\rm Peri}, 6}_i$ by~(\ref{eq:Peri_6_appendix}), and
\begin{align}
   \nabla_{\tilde{\bfx}_i^{{\rm Peri}, 4}} \tilde{\bfx}_{i+1}^{{\rm Peri}} &= \bfI 
   +  \begin{bmatrix}
\nabla {\rm LN}(\bfx^{{\rm Peri}, 4}_{i, 1}) &        &        \\
                 & \ddots &        \\
                 &        & \nabla {\rm LN}(\bfx^{{\rm Peri}, 4}_{i, n})
\end{bmatrix} 
\begin{bmatrix}
c_1\bfW_{i}^{(1)} &        &        \\
                 & \ddots &        \\
                 &        & c_1\bfW_{i}^{(1)}
\end{bmatrix}^\top \\
   & \begin{bmatrix}
\bfh'_1 &        &        \\
                 & \ddots &        \\
                 &        & \bfh'_n
\end{bmatrix}
\begin{bmatrix}
c_2\bfW_{i}^{(2)} &        &        \\
                 & \ddots &        \\
                 &        & c_2\bfW_{i}^{(2)}
\end{bmatrix}^\top \begin{bmatrix}
\nabla {\rm LN}(c_1 c_2 \bfx^{{\rm Peri}, 6}_{i, 1}) &        &        \\
                 & \ddots &        \\
                 &        & \nabla {\rm LN}(c_1 c_2 \bfx^{{\rm Peri}, 6}_{i, n}),
\end{bmatrix} \\
& \text{using Proposition~\ref{prop:grad_LN}, $\nabla {\rm LN}(c_1 c_2 \bfx^{{\rm Peri}, 6}_{i, j})=\frac{1}{c_1 c_2} \nabla {\rm LN}(\bfx^{{\rm Peri}, 6}_{i, j})$, we have} \\
&= \bfI 
   +  \begin{bmatrix}
\nabla {\rm LN}(\bfx^{{\rm Peri}, 4}_{i, 1}) &        &        \\
                 & \ddots &        \\
                 &        & \nabla {\rm LN}(\bfx^{{\rm Peri}, 4}_{i, n})
\end{bmatrix} 
\begin{bmatrix}
\cancel{c_1}\bfW_{i}^{(1)} &        &        \\
                 & \ddots &        \\
                 &        & \cancel{c_1}\bfW_{i}^{(1)}
\end{bmatrix}^\top \\
   & \begin{bmatrix}
\bfh'_1 &        &        \\
                 & \ddots &        \\
                 &        & \bfh'_n
\end{bmatrix}
\begin{bmatrix}
\cancel{c_2}\bfW_{i}^{(2)} &        &        \\
                 & \ddots &        \\
                 &        & \cancel{c_2}\bfW_{i}^{(2)}
\end{bmatrix}^\top \begin{bmatrix} \cancel{\frac{1}{c_1 c_2}}
\nabla {\rm LN}( \bfx^{{\rm Peri}, 6}_{i, 1}) &        &        \\
                 & \ddots &        \\
                 &        & \cancel{\frac{1}{c_1 c_2}} \nabla {\rm LN}(\bfx^{{\rm Peri}, 6}_{i, n})
\end{bmatrix},
\end{align}
which remains the same and verifies the re-scaling invariance.

When large activation occurs in the feed-forward sublayer (\ref{eq:Peri_6_appendix}), i.e., $\bfx^{\rm Peri, 6}_i$ is large, since the input $\bfx^{\rm Peri, 5}_i$  to the sublayer is normalized and bounded in magnitude (Lemma~\ref{lemma:ellipsoid}), at least one of $\bfW_{i}^{(1)}$ or $\bfW_{i}^{(2)}$ must be large. However, the local sensitivity term for the feed-forward layer is invariant under re-scaling of $\bfW_{i}^{(1)}$ and $\bfW_{i}^{(2)}$, thus, it stays at its nominal magnitude.
\end{document}